\documentclass[twoside,11pt]{article}

\usepackage[abbrvbib, preprint]{jmlr2e}

\usepackage[utf8]{inputenc}

\usepackage{graphicx}

\usepackage{color}

\usepackage{hyperref}
\hypersetup{
    final=true, 
    plainpages=false,
    pdfstartview=FitV,
    pdftoolbar=true,
    pdfmenubar=true,
    pdfencoding=auto,
    psdextra,
    bookmarksopen=true,
    bookmarksnumbered=true,
    breaklinks=true,
    linktocpage=true,
 	pdfinfo={
 	    Title={A Bregman Learning Framework for Sparse Neural Networks},
	    Keywords={Bregman Iterations, Sparse Neural Networks, Inverse Scale Space, Optimization}
	    author={Leon Bungert, Tim Roith, Daniel Tenbrinck, Martin Burger}}
}
\usepackage{amssymb,amsmath}

\usepackage[ruled,vlined]{algorithm2e}
\usepackage{setspace}

\usepackage[capitalise,noabbrev,nameinlink]{cleveref}
\crefname{assumption}{Assumption}{Assumptions}

\crefname{example}{Example}{Examples}
\crefname{lemma}{Lemma}{Lemmata}
\crefname{corollary}{Corollary}{Corollaries}
\crefname{proposition}{Proposition}{Propositions}
\crefname{remark}{Remark}{Remarks}
\crefname{definition}{Definition}{Definitions}

\usepackage{booktabs}
\usepackage{multirow}

\usepackage{caption}
\usepackage{subcaption}

\usepackage{enumitem}
\usepackage[margin=1.5in]{geometry}
\usepackage{hhline}
\usepackage{todonotes}
\setuptodonotes{inline}

\usepackage{tabularx}
\newcolumntype{R}{>{\raggedright\arraybackslash}X}
\newcolumntype{C}{>{\centering\arraybackslash}X}
\newcolumntype{L}{>{\raggedleft\arraybackslash}X}

\definecolor{internationalkleinblue}{rgb}{0.0, 0.18, 0.65}

\newcommand{\revision}[1]{{\color{magenta}#1}}
\renewcommand{\revision}[1]{#1}

\newcommand{\comment}[1]{}

\newcommand{\R}{\mathbb{R}}
\newcommand{\N}{\mathbb{N}}

\DeclareMathOperator*{\argmin}{arg\,min}

\newcommand{\prox}[1]{\,\mathrm{prox}_{#1}}
\newcommand{\abs}[1]{\left\vert#1\right\vert}
\newcommand{\norm}[1]{\Vert#1\Vert}
\newcommand{\st}{\,:\,}

\DeclareMathOperator*{\sign}{sign}
\newcommand{\eps}{\varepsilon}
\DeclareMathOperator{\dom}{dom}

\newcommand{\param}{\theta}

\newcommand{\net}{f}     
\newcommand{\inp}{x}

\newcommand{\oup}{y}

\newcommand{\Inp}{\mathcal{X}}
\newcommand{\Oup}{\mathcal{Y}}
\newcommand{\Param}{\Theta}

\newcommand{\trSet}{\mathcal{T}}

\newcommand{\loss}{\ell}
\newcommand{\empBatchLoss}{L}
\newcommand{\empLoss}{\mathcal{L}}

\newcommand{\E}{\mathbb{E}}

\newcommand{\Exp}[1]{\E\left[#1\right]}
\newcommand{\Var}[1]{\mathrm{Var}\left[#1\right]}
\renewcommand{\P}{\mathbb{P}}
\newcommand{\func}{J}
\newcommand{\sg}{p}

\renewcommand{\d}{\mathrm{d}}

\newcommand{\LinBreg}{\mbox{LinBreg}}
\newcommand{\AdaBreg}{\mbox{AdaBreg}}

\newtheorem{assumption}{Assumption}

\numberwithin{equation}{section}

\ShortHeadings{A Bregman Learning Framework}{Bungert et al}
\firstpageno{1}

\begin{document}

\title{A Bregman Learning Framework for\\ Sparse Neural Networks}

\author{\name Leon Bungert \email leon.bungert@hcm.uni-bonn.de  \\
        \addr Hausdorff Center for Mathematics\\
        University of Bonn\\
        Endenicher Allee 62, Villa Maria
        53115 Bonn, Germany 
        \AND
        \name Tim Roith \email tim.roith@fau.de \\
        \addr Department of Mathematics\\
        Friedrich--Alexander University Erlangen--Nürnberg\\
        Cauerstraße 11, 91058 Erlangen, Germany
        \AND
        \name Daniel Tenbrinck \email daniel.tenbrinck@fau.de\\
        \addr Department of Mathematics\\
        Friedrich--Alexander University Erlangen--Nürnberg\\
        Cauerstraße 11, 91058 Erlangen, Germany
        \AND
        \name Martin Burger \email martin.burger@fau.de\\
        \addr Department of Mathematics\\
        Friedrich--Alexander University Erlangen--Nürnberg\\
        Cauerstraße 11, 91058 Erlangen, Germany
       }

\editor{}

\maketitle

\begin{abstract}
We propose a learning framework based on stochastic Bregman iterations\revision{, also known as mirror descent,} to train sparse neural networks with an inverse scale space approach.
We derive a baseline algorithm called \emph{LinBreg}, an accelerated version using momentum, and \emph{AdaBreg}, which is a Bregmanized generalization of the \emph{Adam} algorithm.
In contrast to established methods for sparse training the proposed family of algorithms constitutes a regrowth strategy for neural networks that is solely optimization-based without additional heuristics. 
Our Bregman learning framework starts the training with very few initial parameters, successively adding only significant ones to obtain a sparse and expressive network.
The proposed approach is extremely easy and efficient, yet supported by the rich mathematical theory of inverse scale space methods.
We derive a statistically profound sparse parameter initialization strategy and provide a rigorous stochastic convergence analysis of the loss decay and additional convergence proofs in the convex regime.
Using only $3.4\%$ of the parameters of ResNet-18 we achieve $90.2\%$ test accuracy on CIFAR-10, compared to $93.6\%$ using the dense network.
Our algorithm also unveils an autoencoder architecture for a denoising task.
The proposed framework also has a huge potential for integrating sparse backpropagation and resource-friendly training.
\end{abstract}

\begin{keywords}
  Bregman Iterations, Sparse Neural Networks, Sparsity, Inverse Scale Space, Optimization
\end{keywords}

\section{Introduction}
Large and deep neural networks have shown astonishing results in challenging applications, ranging from real-time image classification in autonomous driving, over assisted diagnoses in healthcare, to surpassing human intelligence in highly complex games \citep{amato2013artificial,rawat2017deep,silver2016mastering}.
The main drawback of many of these architectures is that they require huge amounts of memory and can only be employed using specialised hardware, like GPGPUs and TPUs.
This makes them inaccessible to normal users with only limited computational resources on their mobile devices or computers \citep{hoefler2021sparsity}. 
Moreover, the carbon footprint of training large networks has become an issue of major concern recently \citep{dhar2020carbon}, hence calling for resource-efficient methods.

The success of large and deep neural networks is not surprising as it has been predicted by universal approximation theorems \citep{cybenko1989approximation,lu2017expressive}, promising a smaller error with increasing number of neurons and layers.
Besides the increase in computational complexity, each neuron added to the network architecture also adds to the amount of free parameters and local optima of the loss.

Consequently, a significant branch of modern research aims for training ``sparse neural networks'', which has lead to different strategies, based on neglecting small parameters or such with little influence on the network output, see \citet{hoefler2021sparsity} for an extensive review.
Apart from computational and resource efficiency, sparse training also sheds light on neural architecture design and might answer the question why certain architectures work better than others.

A popular approach for generating sparse neural networks are pruning techniques \citep{lecun1990optimal,han2015learning}, which have been developed to sparsify a dense neural network during or after training by dropping dispensable neurons and connections.
Another approach, which is based on the classical Lasso method from compressed sensing \citep{tibshirani1996regression}, incorporates $\ell_1$ regularization into the training problem, acting as convex relaxation of sparsity-enforcing $\ell_0$ regularization.
These endeavours are further supported by the recently stated ``lottery ticket hypothesis'' \citep{frankle2018lottery}, which postulates that dense, feed-forward networks contain sub-networks with less neurons that, if trained in isolation, can achieve the same test accuracy as the original network.

An even more intriguing idea is ``grow-and-prune'' \citep{dai2019nest}, which starts with a sparse network and augments it during training, while keeping it as sparse as possible.
To this end new neurons are added, e.g., by splitting overloaded neurons into new specimen or using gradient-based indicators, while insignificant parameters are set to zero by thresholding.

Many of the established methods in the literature are bound to specific architectures, e.g., fully-connected feedforward layers \citep{castellano1997iterative,liu2021}.
In this paper we propose a more conceptual and optimization-based approach.
The idea is to mathematically follow the intuition of starting with very few parameters and adding only necessary ones in an inverse scale space manner, see \cref{fig:kernels} for an illustration of our algorithm on a convolutional neural network.
For this sake we propose a Bregman learning framework utilizing linearized Bregman iterations---originally introduced for compressed sensing by \citet{yin2008bregman}---for training sparse neural networks.

Our \textbf{main contributions} are the following:
\begin{itemize}
\setlength\itemsep{1pt}
    \item We derive an extremely simple and efficient algorithm for training sparse neural networks, called \emph{\LinBreg{}}.
    \item We also propose a momentum-based acceleration and \emph{\AdaBreg{}}, which utilizes the Adam algorithm \citep{kingma2014adam}.
    \item We perform a rigorous stochastic convergence analysis of \LinBreg{} \revision{for strongly convex losses, in infinite dimensions, and without any smoothness assumptions on $\func$.}
    \item We propose a sparse initialization strategy for the network parameters.
    \item We show that our algorithms are effective for training sparse neural networks and show their potential for architecture design by unveiling a denoising autoencoder.
\end{itemize}

\begin{figure}[tb]
\def\PicWidth{0.24\textwidth}%
\centering%
\begin{subfigure}[b]{\PicWidth}
\includegraphics[width=\textwidth,trim=0.35cm 0.35cm 0.2cm 0.35cm,clip]{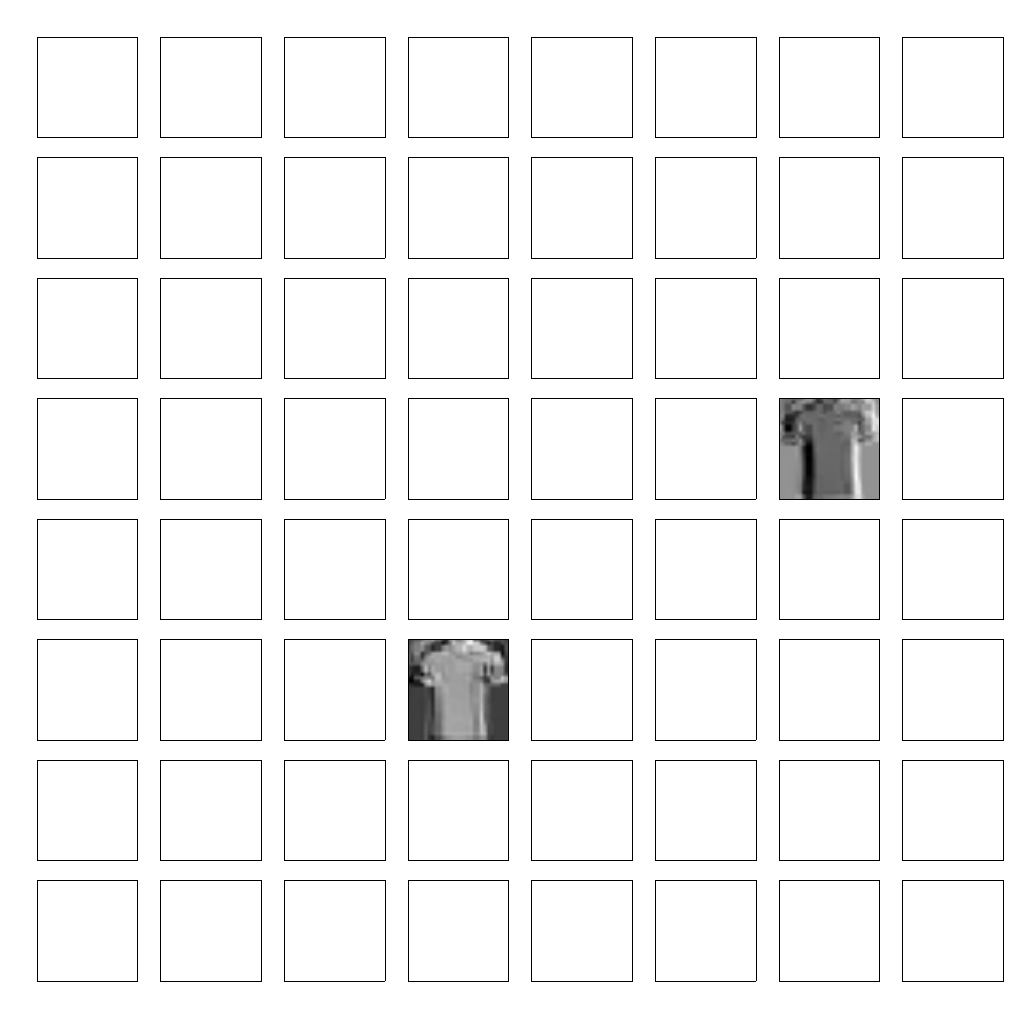}%
\caption{Iteration 0}
\end{subfigure}
\hfill%
\begin{subfigure}[b]{\PicWidth}
\includegraphics[width=\textwidth,trim=0.35cm 0.35cm 0.2cm 0.35cm,clip]{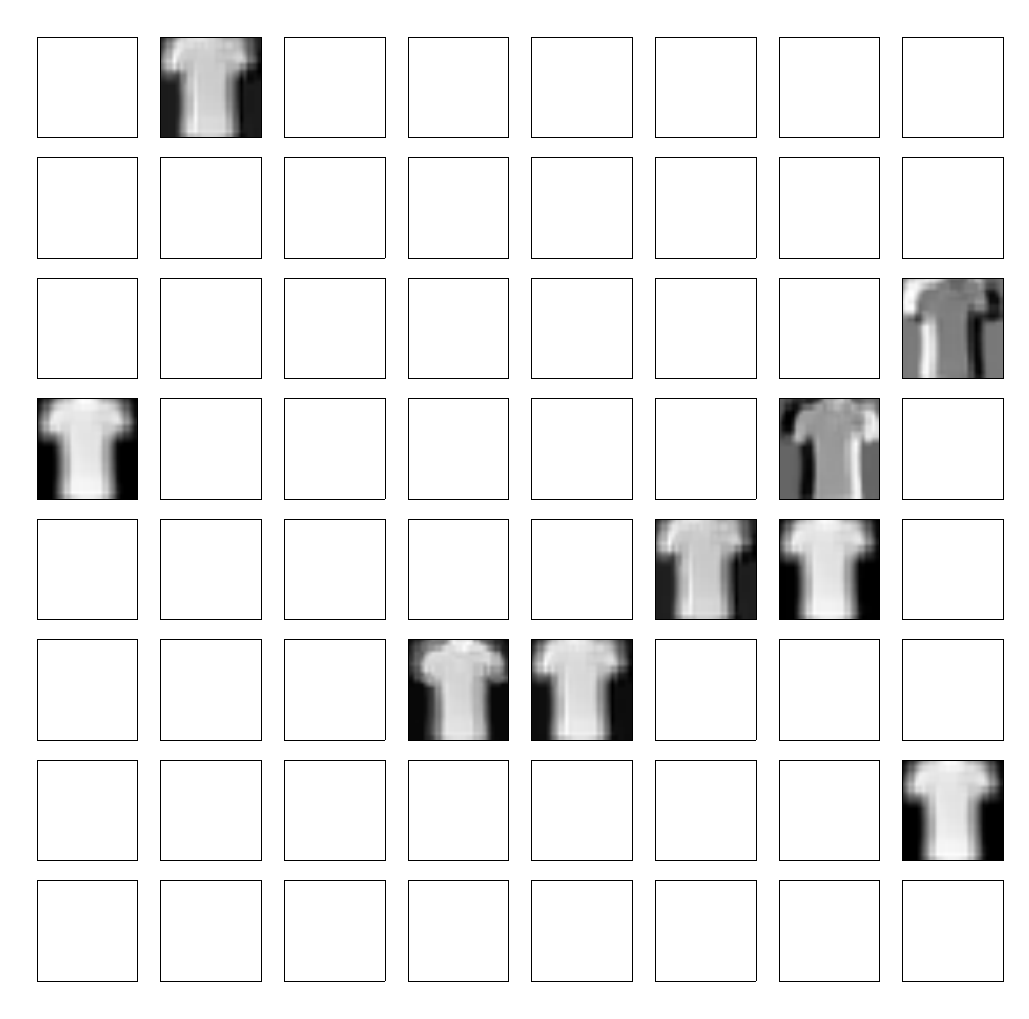}%
\caption{Iteration 5}
\end{subfigure}
\hfill%
\begin{subfigure}[b]{\PicWidth}
\includegraphics[width=\textwidth,trim=0.35cm 0.35cm 0.2cm 0.35cm,clip]{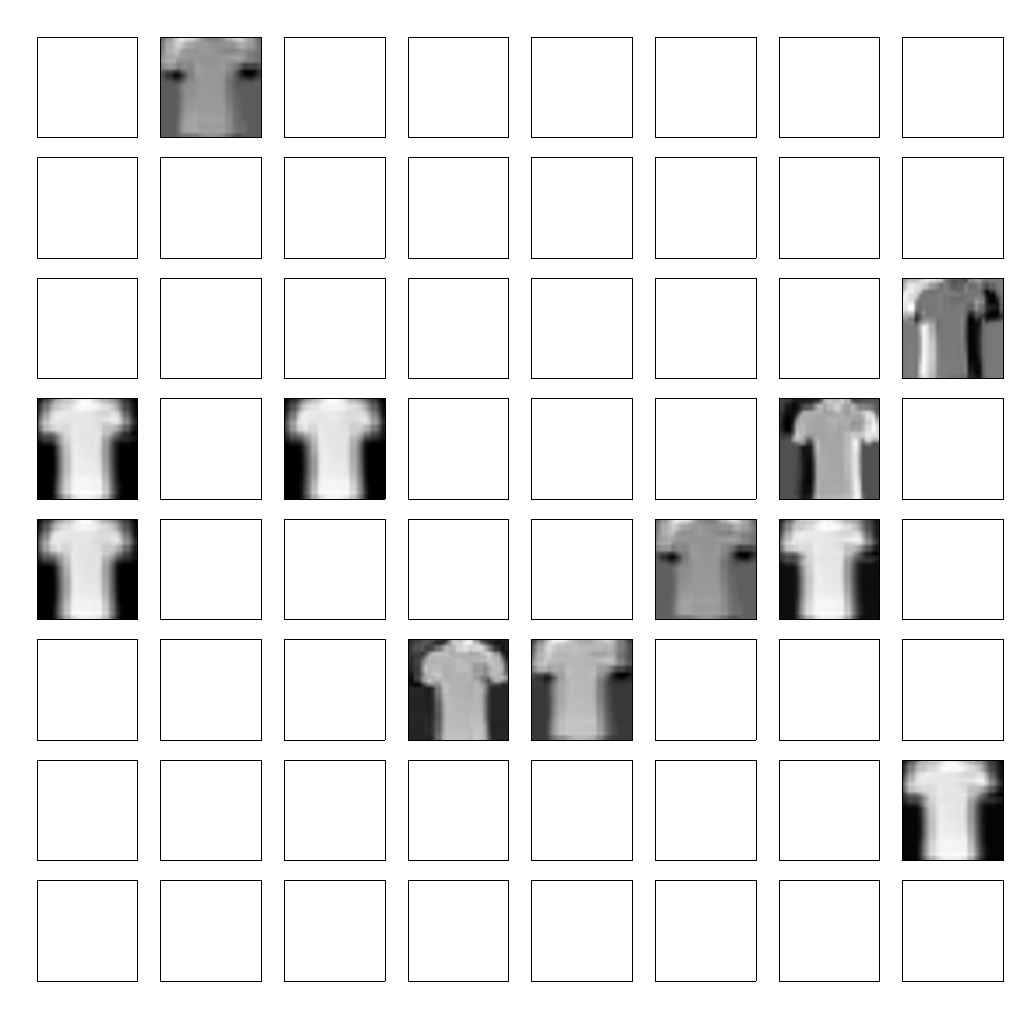}%
\caption{Iteration 20}
\end{subfigure}
\hfill%
\begin{subfigure}[b]{\PicWidth}
\includegraphics[width=\textwidth,trim=0.35cm 0.35cm 0.2cm 0.35cm,clip]{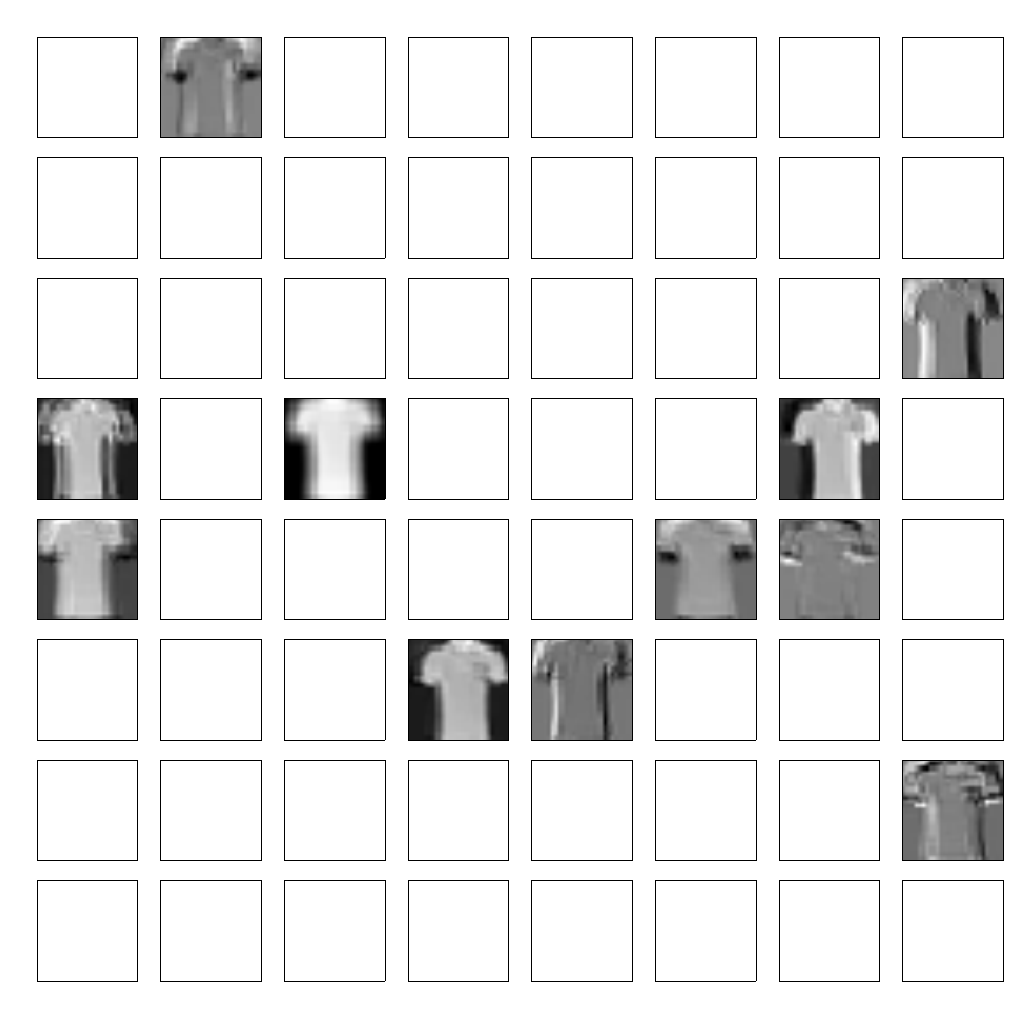}%
\caption{Iteration 100}
\end{subfigure}
\caption{Inverse scale space character of \LinBreg{} visualized through feature maps of a convolutional neural network. 
Descriptive kernels are gradually added in the training process.}
\label{fig:kernels}
\vspace{-10pt}
\end{figure}

The structure of this paper is as follows:
In \cref{sec:nutshell} we explain our baseline algorithm \emph{LinBreg} in a nutshell and in \cref{sec:related_work} we discuss related work.
\cref{sec:prelim_nn,sec:prelim_conv_ana} clarify notation and collect preliminaries on neural networks and convex analysis, the latter being important for the derivation and analysis of our algorithms.
In \cref{sec:bregman_training} we explain how Bregman iterations can be incorporated into the training of sparse neural networks, derive and discuss variants of the proposed Bregman learning algorithm, including accelerations using momentum and Adam.
We perform a mathematical analysis for stochastic linearized Bregman iterations in \cref{sec:analysis} and discuss conditions for convergence of the loss function and the parameters.
In \cref{sec:experiments} we first discuss our statistical sparse initialization strategy and then evaluate our algorithms on benchmark data sets (MNIST, Fashion-MNIST, CIFAR-10) using feedforward, convolutional, and residual neural networks.

\subsection{The Bregman Training Algorithm in a Nutshell}
\label{sec:nutshell}
\begin{algorithm}[t!]
\def\commentWidth{5cm}
\newcommand{\atcp}[1]{\tcp*[r]{\makebox[\commentWidth]{#1\hfill}}}
\setstretch{1.2}
\DontPrintSemicolon
\SetKwInOut{Input}{input}\SetKwInOut{Output}{output}
\SetKwInOut{Default}{default}\SetKwInOut{Default}{default}
\revision{\Default{$\delta=1$}}
$\param\gets$ \cref{sec:initialization},\quad $v \gets \partial\func(\param) + \frac{1}{\delta}\param$ \atcp{initialize}
\For{\upshape{epoch} $e = 1$ \KwTo $E$}{
\For{\upshape{minibatch} $B\subset \trSet$}{
$g \gets \nabla\empBatchLoss(\param;B)$ \atcp{Backpropagation}
$v \gets v - \tau g$ \atcp{Gradient step}
$\param \gets \prox{\delta\func}\left(\delta v\right)$ \atcp{Regularization}
}
}
\caption{\emph{\LinBreg{}}, an inverse scale space algorithm for training sparse neural networks by successively adding weights whilst minimizing the loss.
The functional $\func$ is sparsity promoting, e.g., the $\ell_1$-norm.
}
\label{alg:proximal_bregman_training}
\end{algorithm}
\cref{alg:proximal_bregman_training} states our baseline algorithm \emph{\LinBreg{}} for training sparse neural networks with an inverse scale space approach. 
Mathematical tools and derivations of \LinBreg{} and its variants \emph{\LinBreg{} with momentum} (\cref{alg:proximal_bregman_training_momentum}) and \emph{\AdaBreg{}} (\cref{alg:proximal_bregman_training_adam}), a generalization of \emph{Adam} \citep{kingma2014adam}, are presented in \cref{sec:bregman_training}; 
a convergence analysis is provided in \cref{sec:analysis}.

\LinBreg{} can easily be applied to any neural network architecture $\net_\param$, parametrized with parameters $\param\in\Param$, using a set of training data $\trSet$, and an empirical loss function $\empBatchLoss(\param;B)$, where $B\subset\trSet$ is a batch of training data.
\LinBreg{}'s most important ingredient is a sparsity enforcing functional $\func:\Param\to(-\infty,\infty]$, which acts on groups of network parameters as, for instance, convolutional kernels, weight matrices, biases, etc.
Following \citet{scardapane2017group} and denoting the collection of all parameter groups for which sparsity is desired by $\mathcal{G}$, two possible regularizers which induce sparsity or group sparsity, respectively, can be defined as
\begin{alignat}{2}
    \label{eq:1-norm}
    \func(\param) &= \lambda\sum_{\mathbf{g}\in\mathcal{G}} \norm{\mathbf{g}}_1,\qquad &&\text{the $\ell_1$-norm},\\
    \label{eq:1-2-norm}
    \func(\param) &= \lambda\sum_{\mathbf{g}\in\mathcal{G}} \sqrt{n_{\mathbf{g}}}\,\norm{\mathbf{g}}_2,\qquad &&\text{the group $\ell_{1,2}$-norm}.
\end{alignat}
Here $\lambda>0$ is a parameter controlling the regularization strength, 
$n_{\mathbf{g}}$ denotes the number of elements in $\mathbf{g}$, and the factor $\sqrt{n_\mathbf{g}}$ ensures a uniform weighting of all groups \citep{scardapane2017group}.

\LinBreg{} uses two variables $v$ and $\param$, coupled through the condition that $v\in\partial\func_\delta(\param)$ is a subgradient of the \emph{elastic net regularization}  $\func_\delta(\param):=\func(\param)+\tfrac{1}{2\delta}\norm{\param}^2$ introduced by \citet{zou2005regularization} (see \cref{sec:prelim_nn,sec:prelim_conv_ana} for definitions).
The algorithm successively updates $v$ with gradients of the loss and recovers sparse parameters $\param$ by applying a proximal operator.
For instance, if $\func(\param)=\lambda\norm{\param}_1$ equals the $\ell_1$-norm, the proximal operator in \cref{alg:proximal_bregman_training} coincides with the soft shrinkage operator:
\begin{align}
    \prox{\delta\func}(\delta v) = \delta\operatorname{shrink}(v;\lambda) := \delta\sign(v)\max(|v|-\lambda,0).
\end{align}
In this case only those parameters $\param$ will be non-zero whose subgradients $v$ have magnitude larger than the regularization parameter $\lambda$.
Furthermore, $\delta>0$ only steers the magnitude of the resulting weights and not their support. 
\revision{Furthermore, if $\func(\param)=0$ then $\prox{\delta\func}(\delta v)=\delta v$ and therefore \cref{alg:proximal_bregman_training} coincides with stochastic gradient descent (SGD) with learning rate $\delta\tau$.
These two observations explain our default choice of $\delta=1$.}

In general, the proximal operators of the regularizers above can be efficiently evaluated since they admit similar closed form solutions based on soft thresholding.
Hence, the computational complexity of \LinBreg{} is dominated by the backpropagation and coincides with the complexity of vanilla stochastic gradient descent.
However, note that our framework has great potential for complexity reduction via sparse backpropagation methods, cf. \citet{dettmers2019sparse}.

The special feature which tells \LinBreg{} apart from standard sparsity regularization \citep{louizos2017learning,scardapane2017group,srinivas2017training} or pruning \citep{lecun1990optimal,han2015learning} is its inverse scale space character.
\LinBreg{} is derived based on Bregman iterations, originally developed for scale space approaches in imaging \citep{osher2005iterative,burger2006nonlinear,yin2008bregman,cai2009linearized,cai2009convergence,zhang2011unified}.
Instead of removing weights from a dense trained network, it starts from a very sparse initial set of parameters (see \cref{sec:initialization}) and successively adds non-zero parameters whilst minimizing the loss.

\subsection{Related Work}\label{sec:related_work}

\paragraph{Dense-to-Sparse Training}

A well-established approach for training sparse neural network consists in solving the regularized empirical risk minimization
\begin{align}\label{eq:reg_emp_risk}
    \min_{\param\in\Param} L(\param;B) + \func(\param),
\end{align}
where $\func$ is a (sparsity-promoting) non-smooth regularization functional.
If $\func$ equals the $\ell_1$-norm this is referred to as \emph{Lasso} \citep{tibshirani1996regression} and was extended to \emph{Group Lasso} for neural networks by \citet{scardapane2017group} by using group norms.
We refer to \citet{de2020sparsity} for a mean-field analysis of this approach.
The regularized risk minimization \labelcref{eq:reg_emp_risk} is a special case of Dense-to-Sparse training.
Even if the network parameters are initialized sparsely, any optimization method for \labelcref{eq:reg_emp_risk} will instantaneously generate dense weights, which are subsequently sparsified.
A different strategy for Dense-to-Sparse training is \emph{pruning} \citep{lecun1990optimal,han2015learning}, see also \citet{zhu2017prune}, which first trains a network and then removes parameters to create sparse weights.
This procedure can also be applied alternatingly, which is referred to as \emph{iterative pruning} \citep{castellano1997iterative}.
The weight removal can be achieved based on different criteria, e.g., their magnitude or their influence on the network output.

\paragraph{Sparse-to-Sparse Training}
In contrast, Sparse-to-Sparse training aims to grow a neural network starting from a sparse initialization until it is sufficiently accurate. 
This is also the paradigm of our \LinBreg{} algorithm, generating an inverse sparsity scale space. 
Other approaches from literature are grow-and-prune strategies \citep{mocanu2018scalable,dettmers2019sparse,dai2019nest,liu2021,Evci2020} which, starting from sparse networks, successively add and remove neurons or connections while training the networks.

\paragraph{Proximal Gradient Descent}

A related approach to \LinBreg{} is \emph{proximal gradient descent} (ProxGD) for optimizing the regularized empirical risk minimization \labelcref{eq:reg_emp_risk}, which is an inherently non-smooth optimization problem due to the presence of the $\ell_1$-norm-type functional $\func$.
Therefore, proximal gradient descent alternates between a gradient step of the loss with a proximal step of the regularization:
\begin{subequations}\label{eq:proxGD}
\begin{align}
    g &\gets \nabla \empBatchLoss(\param;B) \\
    \param &\gets \param - \tau g \\
    \param &\gets \prox{\tau\func}(\param).
\end{align}
\end{subequations}
Applications for training neural networks and convergence analysis of this algorithm and its variants can be found, e.g., in \citet{nitanda2014stochastic,rosasco2014convergence,reddi2016proximal,yang2019proxsgd,yun2020general}.
It differs from \cref{alg:proximal_bregman_training} by the lack of a subgradient variable and by using the learning rate $\tau$ within the proximal operator.
These seemingly minor algorithmic differences cause major differences for the trained parameters.
Indeed, the effect of $\func$ kicks in only after several iterations when the proximal operator has been applied sufficiently often to set some parameters to zero, \revision{as can be observed in \cref{fig:fc_net} below.}
Furthermore, proximal gradient descent does not decrease the loss monotonously which we are able to prove for \LinBreg{}.

\paragraph{Bregman Iterations}
Bregman iterations and in particular linearized Bregman iterations have been introduced and thoroughly analyzed for sparse regularization approaches in imaging and compressed sensing (see, e.g., \citet{osher2005iterative,yin2008bregman,bachmayr2009iterative,cai2009linearized,cai2009convergence,yin2010analysis,burger2007inverse,burger2013adaptive}).
\revision{More recent applications of Bregman type methods in the context of image restoration are \citet{benfenati2013inexact,jia2016image,li2018adaptive}.}
Linearized Bregman iterations for non-convex problems, which appear in machine learning and imaging applications like blind deblurring, have first been analyzed by \citet{bachmayr2009iterative,benning2018modern,benning2018choose}.
\citet{benning2018modern} also showed that linearized Bregman iterations for convex problems can be formulated as forward pass of a neural network.
\citet{benning2018choose} applied them for training neural networks with low-rank weight matrices, using nuclear norm regularization.
\citet{huang2016split} suggested a split Bregman approach for training sparse neural networks and \citet{fu2019exploring} provided a deterministic convergence result along the lines of \citet{benning2018choose}.
A recent analysis of \revision{Bregman stochastic gradient descent, which is the same as linearized Bregman iterations,} however using strong regularity assumptions on the involved functions, is done by \citet{dragomir2021fast}.

\paragraph{Mirror Descent}
\revision{%
As it turns out, linearized Bregman iterations are largely known under yet another name: \emph{mirror descent}. 
This method was first proposed by \citet{nemirovskij1983problem} and related to Bregman distances by \citet{beck2003mirror}.
\citet{dragomir2021fast} present a literature overview of stochastic mirror descent.
Some months after the release of the preprint version of the present article, \citet{dorazio2021stochastic} presented a convergence analysis of stochastic mirror descent a.k.a. Bregman iterations, using a weaker bounded variance condition for the stochastic gradients, albeit working in a smooth setting.
In contrast, our analysis does not require any smoothness of $\func$.}

\subsection{Preliminaries on Neural Networks}
\label{sec:prelim_nn}
We denote neural networks, which map from an input space $\Inp$ to an output space $\Oup$ and have parameters in some parameter space $\Param$, by
\begin{align}
    \net_\param:\Inp\to\Oup, \quad\param\in\Param.
\end{align}
In principle $\Inp$, $\Oup$, and $\Param$ can be infinite-dimensional and we only assume that $\Param$ is a Hilbert space, equipped with an inner product $\langle\tilde\param,\param\rangle$ and associated norm $\norm{\param}=\sqrt{\langle\param,\param\rangle}$.
Given a set of training pairs $\trSet\subset\Inp\times\Oup$ and a loss function $\loss:\Oup\times\Oup\to\R$ we denote the empirical loss associated to the training data by
\begin{align}
\empLoss(\param) := \frac{1}{|\trSet|}\sum_{(\inp,\oup)\in \trSet}\loss(\net_\param(\inp),\oup).
\end{align}
The empirical risk minimization approach to finding optimal parameters $\param\in\Param$ of the neural network $\net_\param$ then consists in solving
\begin{align}\label{eq:empir_risk_min}
    \min_{\param\in\Param} \empLoss(\param).
\end{align}
If one assumes that the training set $\trSet$ is sampled from some probability measure $\rho$ on the product space $\Inp\times\Oup$, the empirical risk minimization is an approximation of the infeasible population risk minimization 
\begin{align}
    \min_{\param\in\Param} \int_{\Inp\times\Oup} \loss(\net_\param(\inp),\oup)~\d\rho(x,y).
\end{align}
One typically samples batches $B\subset\trSet$ from the training set and replaces $\empLoss(\param)$ by the empirical risk of the batch
\begin{align}\label{eq:emp_loss}
    \empBatchLoss(\theta;B):=\frac{1}{|B|}\sum_{(\inp,\oup)\in B}\loss(\net_\param(\inp),\oup),
\end{align}
which is utilized in \emph{stochastic} gradient descent methods.

For a feed-forward architecture with $L\in\N$ layers of sizes $n_l$ 
we split the variable $\param$ into weights and 
biases $W^l\in\R^{n_{l},n_{l-1}}$, $b^l\in\R^{n_{l}}$ for 
$l\in\{1,\ldots, L\}$. 
In this case we have
\begin{align}\label{eq:L-layer_net}
\net_\param(\inp) = \Phi^L \circ \dots \circ \Phi^1(x),
\end{align}
where the $l$-th layer for $l\in\{1,\dots,L\}$ is given by 
\begin{align}
\Phi^{l}(z) :=\sigma^l(W^l z + b^l).
\end{align}
Here $\sigma^l$ denote activation functions, as for instance ReLU, TanH, Sigmoid, etc., \citep{Goodfellow16}.
In this case, sparsity promoting regularizers are the $\ell_1$-norm or the group $\ell_{1,2}$-norm
\begin{align}
    \label{eq:1_norm_ffwd}
    \func(\param) &= \lambda\sum_{l=1}^L \norm{W^l}_{1,1},\\
    \label{eq:1-2_norm_ffwd}
    \func(\param) &= \lambda\sum_{l=1}^L \sqrt{n_{l-1}}~\norm{W^l}_{1,2},
\end{align}
which induce sparsity of the weight matrices and of the non-zero rows of weight matrices, respectively. 
Here the scaling $\sqrt{n_{l-1}}$ weighs the influence of the $l$-th layer based on the number of incoming neurons.

\subsection{Preliminaries on Convex Analysis}
\label{sec:prelim_conv_ana}
In this section we introduce some essential concepts from convex analysis which we need to derive \LinBreg{} and its variants and in order to make our argumentation more self-contained. \revision{For an overview of these topics we refer to \citet{benning2018modern, rockafellar1997convex, bauschke2011convex}.}
\revision{A functional $\func:\Param\to(-\infty,\infty]$ on a Hilbert space $\Param$ is called convex if 
\begin{align}
    \func(\lambda\overline{\param}+(1-\lambda)\param)\leq
    \lambda\func(\overline{\param})
    +(1-\lambda)\func(\param),\quad\forall\lambda\in[0,1],\,\overline{\param},\param\in\Param.
\end{align}
We define the effective domain of $\func$ as $\dom(\func):=\{\param\in\Param\st\func(\param)\neq\infty\}$ and call $\func$ proper if $\dom(\func)\neq\emptyset$.
Furthermore, $\func$ is called lower semicontinuous if $\func(u)\leq\liminf_{n\to\infty}\func(u_n)$ holds for all sequences $(u_n)_{n\in\N}\subset\Param$ converging to $u$.
}
First, we define the subdifferential of a convex \revision{and proper} functional $\func:\Param\to(-\infty,\infty]$ at a point $\param\in\Param$ as
\begin{align}
\label{eq:subgrad}
    \partial\func(\param) := \left\lbrace \sg\in\Param \st \func(\param) + \langle\sg,\overline{\param}-\param \rangle \leq \func(\overline{\param}),\;\forall\overline{\param}\in\Param\right\rbrace.
\end{align}
The subdifferential is a non-smooth generalization of the derivative and coincides with the classical gradient (or Fr\'echet derivative) if $\func$ is differentiable.
\revision{We denote $\dom(\partial\func):=\{\param\in\Param\st\partial\func(\param)\neq\emptyset\}$ and observe that $\dom(\partial\func)\subset\dom(\func)$.}

Next, we define the Bregman distance of two points $\param\in\dom(\partial\func),\overline{\param}\in\Param$ with respect to a convex \revision{and proper} functional $\func:\Param\to(-\infty,\infty]$ as
\begin{align}\label{eq:bregman_distance}
    D^\sg_\func(\overline{\param},\param) := \func(\overline \param)-\func({\param})-\langle\sg,\overline\param-{\param}\rangle,\quad\sg\in\partial\func(\param).
\end{align}
The Bregman distance can be interpreted as the distance between the linearization of $\func$ at $\param$ and its graph and hence somewhat measures the degree of linearity of the functional.
Note furthermore that the Bregman distance \labelcref{eq:bregman_distance} is neither definite, symmetric nor fulfills the triangle inequality, hence it is not a metric. 
However, it fulfills the two distance axioms
\begin{align}
    D^\sg_\func(\overline{\param},\param) \geq 0,\quad D^\sg_\func(\param,\param)=0,\quad\forall\overline{\param}\in\Param,\param\in\dom(\partial\func).
\end{align}
By summing up two Bregman distances, one can also define the symmetric Bregman distance with respect to $\sg\in\partial\func({\param})$ and $\overline\sg\in\partial\func(\overline{\param})$ as
\begin{align}
    D^\mathrm{sym}_\func(\overline{\param},\param) := D^\sg_\func(\overline{\param},\param) + D^{\overline{\sg}}_\func(\param,\overline\param).
\end{align}
Here, we suppress the dependency on $\sg$ and $\overline{\sg}$ to simplify the notation.

Last, we define the proximal operator of a \revision{convex, proper and lower semicontinuous functional} $\func:\Param\to(-\infty,\infty]$ as
\begin{align}
    \prox{\func}(\overline{\param}) := \argmin_{\param\in\Param} \frac{1}{2}\norm{\param-\overline{\param}}^2 + \func(\param).
\end{align}
Proximal operators are a key concept in non-smooth optimization since they can be used to replace gradient descent steps of non-smooth functionals, as done for instance in proximal gradient descent \labelcref{eq:proxGD}.
Obviously, given some $\overline{\param}\in\Param$ the proximal operator outputs a new element $\param\in\Param$ which has a smaller value of $\func$ whilst being close to the previous element $\overline{\param}$.

\section{Bregmanized training of Neural Networks}
\label{sec:bregman_training} 

In this section we first give a short overview of inverse scale space flows which are the time-continuous analogue of our algorithms. 
Subsequently, we derive \emph{\LinBreg{}} (\cref{alg:proximal_bregman_training}) by passing from Bregman iterations to linearized Bregman iterations, which we then reformulate in a very easy and compact form.
We then derive \emph{\LinBreg{} with momentum} (\cref{alg:proximal_bregman_training_momentum}) by discretizing a second-order in time inverse scale space flow and propose \emph{AdaBreg} (\cref{alg:proximal_bregman_training_adam}) as a generalization of the popular Adam algorithm \citep{kingma2014adam}.

\subsection{Inverse Scale Space Flows (with Momentum)}
\label{sec:iss_flows}

In the following we discuss the inverse scale space flow, which arises as gradient flow of a loss functional $\empLoss$ with respect to the Bregman distance \labelcref{eq:bregman_distance}.
In particular, it couples the minimization of $\empLoss$ with a simultaneous regularization through $\func$.
To give meaning to this, one considers the following implicit Euler scheme
\begin{subequations}\label{eq:bregman_iteration}
\begin{align}
    \param^{(k+1)} &= \argmin_{\param\in\Param} D^{\sg^{(k)}}_\func(\param,\param^{(k)}) + \tau^{(k)}\empLoss(\param), \\
    \sg^{(k+1)} &= \sg^{(k)} - \tau^{(k)}\nabla\empLoss(\param^{(k+1)}) \in \partial \func(\param^{(k+1)})
\end{align}
\end{subequations}
which is know as \emph{Bregman iteration}. 
Here, $\param^{(k)}$ is the previous iterate with subgradient $\sg^{(k)}\in\partial \func(\param^{(k)})$, and $\tau^{(k)}>0$ is a sequence of time steps.
Note that the subgradient update in the second line of \labelcref{eq:bregman_iteration} coincides with the optimality conditions of the first line.

The time-continuous limit of \labelcref{eq:bregman_iteration} as $\tau^{(k)}\to 0$ is the inverse scale space flow
\begin{align}\label{eq:iss_0}
    \begin{cases}
        \dot{\sg}_t = - \nabla\empLoss(\param_t), \\
        \sg_t \in \partial \func(\param_t),
    \end{cases}
\end{align}
see \citet{burger2006nonlinear,burger2007inverse} for a rigorous derivation in the context of image denoising.

If $\func(\param)=\frac{1}{2}\norm{\param}^2$ then $\partial\func(\param)=\param$ and \labelcref{eq:iss_0} coincides with the standard gradient flow
\begin{align}
    \dot{\param}_t = -\nabla\empLoss(\param_t).
\end{align}
Hence, the inverse scale space is a proper generalization of the gradient flow and allows for regularizing the path along which the loss is minimized using $\func$ (see \citet{benning2018choose}). 
For strictly convex loss functions this might seem pointless since they have a unique minimum anyways, however, for merely convex or even non-convex losses the inverse scale space allows to `select' (local) minima with desirable properties.

In this paper, we also propose an inertial version of \labelcref{eq:iss_0} which depends on an inertial parameter $\gamma\geq 0$ and takes the form
\begin{align}\label{eq:iss_inertial}
    \begin{cases}
        \gamma\ddot{\sg}_t + \dot{\sg}_t = - \nabla\empLoss(\param_t), \\
        \sg_t \in \partial \func(\param_t).
    \end{cases}
\end{align}
One can introduce the momentum variable $m_t:=\dot{\sg}_t$ which solves the differential equation 
\begin{align*}
    \gamma \dot{m}_t + m_t = -\nabla \empLoss(\param_t).
\end{align*}
If one assumes $m_0=0$, this equation has the explicit solution
\begin{align*}
    m_t=-\int_0^t\exp\left(\frac{s-t}{\gamma}\right)\nabla\empLoss(\param_s)\d s
\end{align*}
and hence the second-order in time equation \labelcref{eq:iss_inertial} is equivalent to the gradient memory inverse scale space flow
\begin{align}\label{eq:iss_gradient_memory}
    \begin{cases}
        \dot{\sg}_t = - \int_0^t \exp\left(\frac{s-t}{\gamma}\right)\nabla\empLoss(\param_s)\d s, \\
        \sg_t \in \partial \func(\param_t).
    \end{cases}
\end{align}
For a nice overview over the derivation of gradient flows with momentum we refer to \citet{orvieto2020role}

\subsection{From Bregman to Linearized Bregman Iterations}\label{sec:derivation}
The starting point for the derivation of \cref{alg:proximal_bregman_training} is the Bregman iteration \labelcref{eq:bregman_iteration}, which is the time discretization of the inverse scale space flow \labelcref{eq:iss_0}.

Since the iterations~\labelcref{eq:bregman_iteration} require the minimization of the loss in every iteration they are not feasible for large-scale neural networks.
Therefore, we consider linearized Bregman iterations \citep{cai2009linearized}, which linearize the loss function by
\begin{align*}
    \empLoss(\param) \approx \empLoss(\param^{(k)}) + \left\langle g^{(k)},\param-\param^{(k)}\right\rangle,\quad g^{(k)}:=\nabla\empLoss(\param^{(k)}),
\end{align*}
and replace the energy $\func$ with the \revision{strongly convex} elastic-net regularization
\begin{align}\label{eq:elastic_net}
    \func_\delta(\param):=\func(\param)+\frac{1}{2\delta}\norm{\param}^2,\quad\revision{\delta\in(0,\infty).}
\end{align}
Omitting all terms which do not depend on $\param$, the first line of \labelcref{eq:bregman_iteration} then becomes
\begin{align}
    \param^{(k+1)} 
    &= \argmin_{\param\in\Param}\langle\tau^{(k)} g^{(k)},\param\rangle + \func_\delta(\param)-\langle v^{(k)},\param\rangle \notag\\
    &= \argmin_{\param\in\Param}\langle\tau^{(k)} g^{(k)},\param\rangle + \func(\param)+\frac{1}{2\delta}\norm{\param}^2-\langle v^{(k)},\param\rangle \notag\\
    &=\argmin_{\param\in\Param}\frac{1}{2\delta}\norm{\param - \delta\left(v^{(k)}-\tau^{(k)} g^{(k)}\right)}^2 + \func(\param) \notag\\
    \label{eq:prox_iteration}
    &=\prox{\delta\func}\left(\delta\left(v^{(k)}-\tau^{(k)} g^{(k)}\right)\right).
\end{align}
The vector $v^{(k)}\in\partial\func_\delta(\param^{(k)})$ is a subgradient of the functional $\func_\delta$ in the previous iterate.
Using the update
\begin{align*}
    v^{(k+1)} := v^{(k)} - \tau^{(k)} g^{(k)}
\end{align*}
and combining this with \labelcref{eq:prox_iteration} we obtain the compact update scheme
\begin{subequations}\label{eq:lin_breg_it}
\begin{align}
    g^{(k)} &= \nabla \empLoss(\param^{(k)}), \\
    v^{(k+1)} &= v^{(k)} - \tau^{(k)} g^{(k)}, \\
    \param^{(k+1)} &= \prox{\delta\func}\left(\delta v^{(k+1)}\right).
\end{align}
\end{subequations}
This iteration is an equivalent reformulation of linearized Bregman iterations \citep{yin2008bregman,cai2009linearized,cai2009convergence,osher2010fast,yin2010analysis,benning2018modern}, which are usually expressed in a more complicated way.
\revision{
Furthermore, it coincides with the mirror descent algorithm \citep{beck2003mirror} applied to the functional $\func_\delta$.

Note that \cite{yin2010analysis} showed that for quadratic loss functions the elastic-net regularization parameter $\delta>0$ has no influence on the asymptotics of linearized Bregman iterations, if chosen larger than a certain threshold. This effect is referred to as \emph{exact regularization} \citep{friedlander2008exact}.}
The iteration scheme simply computes a gradient descent in the subgradient variable $v$ and recovers the weights $\param$ by evaluating the proximal operator of~$v$.
This makes it significantly cheaper than the original Bregman iterations~\labelcref{eq:bregman_iteration} which require the minimization of the loss \revision{in} every iteration.

Note that the last line in \labelcref{eq:lin_breg_it} is equivalent to $v^{(k+1)}$ satisfying the optimality condition
\begin{align}\label{eq:v_subgrad}
    v^{(k+1)}\in \partial\func_\delta(\param^{(k+1)}).
\end{align}
In particular, by letting $\tau^{(k)}\to 0$ the iteration~\labelcref{eq:lin_breg_it} can be viewed as explicit Euler discretization of the inverse scale space flow \labelcref{eq:iss_0} of the elastic-net regularized functional $\func_\delta$:
\begin{align}\label{eq:iss_elastic_net}
    \begin{cases}
    \dot{v}_t = -\nabla\empLoss(\param_t), \\
    v_t \in \partial \func_\delta(\param_t).
    \end{cases}
\end{align}
By embedding \labelcref{eq:lin_breg_it} into a stochastic batch gradient descent framework we obtain \LinBreg{} from \cref{alg:proximal_bregman_training}.

\subsection{Linearized Bregman Iterations with Momentum}
More generally we consider an inertial version of \labelcref{eq:iss_elastic_net}, which as in \cref{sec:iss_flows} is given by
\begin{align}
    \label{eq:second_order_iss}
    \begin{cases}
        \gamma \ddot{v}_t + \dot{v}_t = -\nabla\empLoss(\param_t), \\
        v_t \in \partial \func_\delta(\param_t).
    \end{cases}
\end{align}
We discretize this equation in time by approximating the time derivatives as
\begin{align*}
    \ddot{v}_t &\approx \frac{v^{(k+1)}-2v^{(k)}+v^{(k-1)}}{(\tau^{(k)})^2}, \\
    \dot{v}_t &\approx \frac{v^{(k+1)}-v^{(k)}}{\tau^{(k)}}, 
\end{align*}
such that after some reformulation we obtain the iteration
\begin{subequations}\label{eq:lin_breg_it_inertia}
\begin{align}
    v^{(k+1)} &= \frac{\tau^{(k)}+2\gamma}{\tau^{(k)}+\gamma} v^{(k)} - \frac{\gamma}{\tau^{(k)}+\gamma}v^{(k-1)} - \frac{(\tau^{(k)})^2}{\tau^{(k)}+\gamma}\nabla\empLoss(\param^{(k)}), \\
    \param^{(k+1)} &= \prox{\delta \func}(\delta v^{(k+1)}).
\end{align}
\end{subequations}
To see the relation to the gradient memory equation \labelcref{eq:iss_gradient_memory}, derived in \cref{sec:iss_flows}, we rewrite \labelcref{eq:lin_breg_it_inertia}, using the new variables
\begin{align}
    m^{(k+1)} := v^{(k)} - v^{(k+1)},\quad
    \beta^{(k)} := \frac{\gamma}{\tau^{(k)}+\gamma}\in[0,1).
\end{align}
Plugging this into \labelcref{eq:lin_breg_it_inertia} yields the iteration
\begin{subequations}\label{eq:lin_breg_it_momentum}
\begin{align}
    m^{(k+1)} &= \beta^{(k)} m^{(k)} + (1-\beta^{(k)})\tau^{(k)} \nabla\empLoss(\param^{(k)}),\\
    v^{(k+1)} &= v^{(k)} - m^{(k+1)},\\
    \param^{(k+1)} &= \prox{\delta\func}(\delta v^{(k+1)}).
\end{align}
\end{subequations}
Similar to before, embedding this into a stochastic batch gradient descent framework we obtain \LinBreg{} with momentum from \cref{alg:proximal_bregman_training_momentum}.
Note that, contrary to stochastic gradient descent with momentum \citep{orvieto2020role}, the momentum acts on the subgradients $v$ and not on the parameters $\param$.

Analogously, we propose \AdaBreg{} in \cref{alg:proximal_bregman_training_adam}, which is a generalization of the Adam algorithm~\citep{kingma2014adam}.
Here, we also apply the bias correction steps on the subgradient $v$ and reconstruct the parameters $\param$ using the proximal operator of the regularizer $\func$.
\begin{algorithm}[htb]
\def\commentWidth{5cm}
\newcommand{\atcp}[1]{\tcp*[r]{\makebox[\commentWidth]{#1\hfill}}}
\setstretch{1.2}
\DontPrintSemicolon
\SetKwInOut{Input}{input}\SetKwInOut{Output}{output}
\SetKwInOut{Default}{default}
\Default{$\delta=1$,\quad $\beta=0.9$}
$\param\gets$ \cref{sec:initialization},\quad $v \gets \partial\func(\param) + \frac{1}{\delta}\param$,\quad $m \gets 0$ \atcp{initialize}
\For{epoch $e = 1$ \KwTo $E$}{
\For{minibatch $B\subset \trSet$}{
$g \gets \nabla\empBatchLoss(\param;B)$      \atcp{Backpropagation}
$m \gets \beta\,m + (1-\beta)\tau\,g$   \atcp{Momentum update}
$v \gets v - m$                         \atcp{Momentum step}
$\param \gets \prox{\delta\func}\left(\delta v\right)$   \atcp{Regularization}
}
}
\caption{\emph{\LinBreg{} with Momentum}, an acceleration of \LinBreg{} using momentum-based gradient memory.}
\label{alg:proximal_bregman_training_momentum}
\end{algorithm}
\begin{algorithm}[htb]
\setstretch{1.2}
\def\commentWidth{5cm}
\newcommand{\atcp}[1]{\tcp*[r]{\makebox[\commentWidth]{#1\hfill}}}
\newcommand{\atcps}[1]{\tcp*[r]{\makebox[3cm]{#1\hfill}}}
\setstretch{1.2}
\DontPrintSemicolon
\SetKwInOut{Input}{input}\SetKwInOut{Output}{output}
\SetKwInOut{Default}{default}
\Default{$\delta=1$, $\beta_1=0.9$,\quad $\beta_2=0.999$,\quad $\epsilon=10^{-8}$}
$\param\gets$ \cref{sec:initialization},\quad $v \gets \partial\func(\param) + \frac{1}{\delta}\param$,\quad $m_1 \gets 0$,\quad $m_2 \gets 0$ \atcps{initialize}
\For{epoch $e = 1$ \KwTo $E$}{
\For{minibatch $B\subset \trSet$}{
$k \gets k+1$\;
$g \gets \nabla\empBatchLoss(\param;B)$ \atcp{Backpropagation}
$m_1 \gets \beta_1\,m_1 + (1-\beta_1)\,g$ \atcp{First moment estimate}
$\hat m_1 \gets m_1/(1-\beta_1^k)$ \atcp{Bias correction}
$m_2 \gets \beta_2\,m_2 + (1-\beta_2)\,g^2$ \atcp{Second raw moment estimate}
$\hat m_2 \gets m_2/(1-\beta_2^k)$ \atcp{Bias correction}
$v\gets v - \tau\,\hat m_1/(\sqrt{\hat m_2} + \epsilon)$ \atcp{Moment step}
$\param \gets \prox{\delta\func}\left(\delta v\right)$ \atcp{Regularization}
}
}
\caption{\emph{\AdaBreg{}}, a Bregman version of the Adam algorithm which uses moment-based bias correction.}
\label{alg:proximal_bregman_training_adam}
\end{algorithm}

\section{Analysis of Stochastic Linearized Bregman Iterations}\label{sec:analysis}

In this section we provide a convergence analysis of stochastic linearized Bregman iterations.
They are valid in a general sense and do not rely on $\empLoss$ being an empirical loss or $\param$ being weights of a neural network.
Still we keep the notation fixed for clarity.
All proofs can be found in the appendix.

We let $(\Omega,F,\P)$ be a probability space, $\Param$ be a Hilbert space, \revision{$\empLoss:\Param\to\R$ a Fr\'echet differentiable loss function, and $g:\Param\times\Omega \to \Param$ an unbiased estimator of $\nabla\empLoss$, meaning $\Exp{g(\param;\omega)} = \nabla\empLoss(\param)$ for all $\param\in\Param$.}
\revision{This and all other expected values are taken with respect to the random variable $\omega\in\Omega$ which, in our case, models the randomly drawn batch of training in data in \labelcref{eq:emp_loss}.}
We study the stochastic linearized Bregman iterations
\begin{subequations}\label{eq:stochlinbreg}
\begin{align}
    \text{draw }&\omega^{(k)}\text{ from }\Omega\text{ using the law of }\P,\\
    g^{(k)} &:= \revision{g(\param^{(k)};\omega^{(k)})},\\
    v^{(k+1)} &:= v^{(k)} - \tau^{(k)} g^{(k)},\\
    \param^{(k+1)} &:= \prox{\delta J}(\delta v^{(k+1)}).
\end{align}
\end{subequations}

For our analysis we need some assumptions on the loss function $\empLoss$ which are very common in the analysis of nonlinear optimization methods.
Besides boundedness from below, we demand differentiability and Lipschitz-continuous gradients, which are standard assumptions in nonlinear optimization since they allow to prove \emph{sufficient decrease} of the loss.
\revision{We refer to \cite{benning2018choose} for an example of a neural network the associated loss of which satisfies the following assumption.}
\begin{assumption}[Loss function]\label{ass:loss}
We assume the following conditions on the loss function:
\begin{itemize}
    \item The loss function $\empLoss$ is bounded from below and without loss of generality we assume $\empLoss\geq 0$.
    \item The function $\empLoss$ is continuously differentiable.
    \item The gradient of the loss function $\param\mapsto\nabla\empLoss(\param)$ is $L$-Lipschitz for $L\in(0,\infty)$:
    \begin{align}\label{ineq:gradLip}
        \norm{\nabla\empLoss(\tilde\param)-\nabla\empLoss(\param)}\leq L \norm{\tilde\param-\param},\quad \forall \param,\tilde\param\in\Param.
    \end{align}
\end{itemize}
\end{assumption}
\begin{remark}
Note that the Lipschitz continuity of the gradient in particular implies the classical estimate \citep{beck2017first,bauschke2011convex}
\begin{align}\label{ineq:L-smooth}
    \empLoss(\tilde\param) \leq \empLoss(\param) + \langle\nabla\empLoss(\param),\tilde\param - \param\rangle + \frac{L}{2}\norm{\tilde\param-\param}^2,\quad\forall\param,\tilde\param\in\Param.
\end{align}
\end{remark}

Furthermore, we need the following assumption, being of stochastic nature, which requires the gradient estimator to have \revision{uniformly} bounded variance.
\begin{assumption}[Bounded variance]\label{ass:variance}
There exists a constant $\sigma>0$ such that for any $\param\in\Param$ it holds
\begin{align}
    \E\left[\norm{g(\param;\omega)-\nabla\empLoss(\param)}^2\right] \leq \sigma^2.
\end{align}
\end{assumption}
\revision{This is a standard assumption in the analysis of stochastic optimization methods and many authors actually demand the more restrictive condition of uniformly bounded stochastic gradients $\Exp{\norm{g(\param;\omega)}^2}\leq C$ for all $\param\in\Param$. 
Remarkably, both assumptions have been shown to be unnecessary for proving convergence of stochastic gradient descent of convex \citep{pmlr-v80-nguyen18c} and non-convex functions \citep{lei2019stochastic}. Generalizing this to linearized Bregman iterations is however completely non-trivial which is why we stick to \cref{ass:variance}.}

\revision{%
We also state our assumptions on the regularizer $\func$, which are extremely mild.
\begin{assumption}[Regularizer]\label{ass:regularizer}
We assume that $\func:\Param\to(-\infty,\infty]$ is a convex, proper, and lower semicontinuous functional on the Hilbert space $\Param$.
\end{assumption}}

All other assumptions will be stated when they are needed.

\subsection{Decay of the Loss Function}
\label{sec:loss_decay}

We first analyze how the iteration \labelcref{eq:stochlinbreg} decreases the loss $\empLoss$.
Such decrease properties of deterministic linearized Bregman iterations in a different formulation were already studied by \citet{benning2018choose,benning2018modern}.
Note that for the loss decay we do not require any sort of convexity of $\empLoss$ whatsoever, but merely Lipschitz continuity of the gradient, i.e., \labelcref{ineq:L-smooth}.

\begin{theorem}[Loss decay]\label{thm:decreasing_loss}
\revision{Assume that \cref{ass:loss,ass:variance,ass:regularizer} hold true, let $\delta>0$,} and let the step sizes satisfy $\tau^{(k)} \leq \frac{2}{\delta L}$.
Then there exist constants $c,C>0$ such that for every $k\in\N$ the iterates of \labelcref{eq:stochlinbreg} satisfy 
\begin{align}\label{ineq:loss_decay}
    \begin{split}
    \E\left[\empLoss(\param^{(k+1)})\right] + \frac{1}{\tau^{(k)}}\E\left[ D^\mathrm{sym}_\func(\param^{(k+1)},\param^{(k)})\right] + \frac{C}{2\delta\tau^{(k)}}\E\left[\norm{\param^{(k+1)}-\param^{(k)}}^2\right] \\
    \leq \E\left[\empLoss(\param^{(k)})\right] + \tau^{(k)}\delta\frac{\sigma^2}{2c},
    \end{split}
\end{align}
\end{theorem}

\begin{corollary}[Summability]\label[corollary]{cor:square_sum}
Under the conditions of \cref{thm:decreasing_loss} and with the additional assumption that the step sizes are non-increasing and square-summable, meaning
\begin{align*}
    \tau^{(k+1)} \leq \tau^{(k)},\quad\forall k\in\N,\qquad
    \sum_{k=0}^{\infty}(\tau^{(k)})^2 <\infty,
\end{align*}
it holds
\begin{align*}
    \sum_{k=0}^\infty \E\left[D^\mathrm{sym}_\func(\param^{(k+1)},\param^{(k)})\right] < \infty, \qquad
    \sum_{k=0}^\infty
    \E\left[\norm{\param^{(k+1)}-\param^{(k)}}^2\right] < \infty.
\end{align*}
\end{corollary}
\begin{remark}[Deterministic case]
Note that in the deterministic setting with $\sigma=0$ the statement of \cref{thm:decreasing_loss} coincides with \citet{benning2018choose,benning2018modern}.
In particular, the loss \revision{decays} monotonously and one gets stronger summability than in \cref{cor:square_sum} since one does not have to multiply with $\tau^{(k)}$.
\end{remark}

\subsection{Convergence of the Iterates}
\label{sec:cvgc_iterates}

In this section we establish two convergence results for the iterates of the stochastic linearized Bregman iterations \labelcref{eq:stochlinbreg}.
According to common practice and for self-containedness we restrict ourselves to \revision{strongly convex} losses.
Obviously, our \revision{results} remain true for non-convex losses if one assumes convexity around local minima and applies our arguments locally.
We note that one could also extend the deterministic convergence proof of \citet{benning2018choose}---which is based on the Kurdyka-\L ojasiewicz (KL) inequality and works for non-convex losses---to the stochastic setting.
However, this is beyond the scope of this paper and conveys less intuition than our proofs.

For our first convergence result---asserting norm convergence of a subsequence---we need \revision{the condition on the loss function \cref{ass:loss},} the bounded variance condition from \cref{ass:variance} and strong convexity of the loss $\empLoss$:
\begin{assumption}[Strong convexity]\label{ass:mu-convex} The loss function $\param\mapsto\empLoss(\param)$ is $\mu$-strongly convex for $\mu\in(0,\infty)$, meaning
\begin{align}
    \empLoss(\tilde\param) \geq \empLoss(\param) + \langle\nabla\empLoss(\param),\tilde\param - \param\rangle + \frac{\mu}{2}\norm{\tilde\param-\param}^2,\quad\forall\param,\tilde\param\in\Param.
\end{align}
\end{assumption}
\revision{Note that by virtue of \labelcref{ineq:L-smooth} it holds $\mu\leq L$ if the loss satisfies \cref{ass:loss,ass:mu-convex}.}

Our second convergence convergence result---asserting convergence in the Bregman distance of $\func_\delta$ which is a stronger topology than norm convergence---requires a stricter convexity condition, tailored to the Bregman geometry.

\begin{assumption}[Strong Bregman convexity]\label{ass:breg-convex}
The loss function $\param\mapsto\empLoss(\param)$ satisfies
\begin{align}
    \empLoss(\tilde\param) \geq \empLoss(\param) + \langle\nabla\empLoss(\param),\tilde\param - \param\rangle + \nu D_{J_\delta}^{v}(\tilde\param,\param),\quad\forall\param,\tilde\param\in\Param,\;v\in\partial J_\delta(\param),
\end{align}
\revision{where $\func_\delta$ for $\delta>0$ is defined in \labelcref{eq:elastic_net}.}
In particular, $\empLoss$ satisfies \cref{ass:mu-convex} with $\mu=\nu/\delta$.
\end{assumption}
\begin{remark}[The Bregman convexity assumption]
\cref{ass:breg-convex} seems to be quite restrictive, however, in finite dimensions it is locally equivalent to \cref{ass:mu-convex}, as we argue in the following.
Note that it suffices if the assumptions above are satisfied in a vicinity of the (local) minimum to which the algorithm converges.
For proving convergence we will use \cref{ass:breg-convex} with $\tilde\param=\param^*$, a (local) minimum, and $\param$ close to $\param^*$.
Using \cref{lem:sum_breg_dist} in the appendix the assumption can be rewritten as
\begin{align*}
    \empLoss(\param^*) \geq \empLoss(\param) + \langle\nabla\empLoss(\param),\param^* - \param\rangle + \frac{\nu}{2\delta}\norm{\param^*-\param}^2 + \nu D_\func^\sg(\param^*,\param),\quad\sg\in\partial\func(\param)
\end{align*}
\revision{
and we will argue that for $\param$ close to $\param^*$ the extra term vanishes, i.e. $D_\func^\sg(\param^*,\param)=0$.
For this we have to show that $\sg$ is not only a subgradient at $\param$ but also at $\param^*$:
If $\sg$ is a subgradient of $\func$ at both points, i.e., $\sg\in \partial\func(\param) \cap \partial\func(\param^*)$ we obtain that their Bregman distance is zero. 
This can be seen using the definition of the Bregman distance \labelcref{eq:bregman_distance}:
\begin{align*}
\left.\begin{aligned}
D_\func^\sg(\param^*,\param) \geq 0\\
D_\func^\sg(\param,\param^*) \geq 0
\end{aligned}\right\rbrace
\implies
0 \leq D_\func^\sg(\param^*,\param) + D_\func^\sg(\param,\param^*) = 0.    
\end{align*}}
In finite dimensions and for $\func(\param)=\norm{\param}_1 = \sum_{i=1}^N|\param_i|$ equal to the $\ell_1$-norm \revision{we can use that $\langle\sg,\param\rangle=\func(\param)=\sum_{i=1}^N\sign(\param_i)\param_i=\sum_{i=1}^N\abs{\param_i}=\func(\param)$} and simplify \revision{the Bregman distance} to
\begin{align*}
    D_\func^\sg(\param^*,\param) = 
    J(\param^*) - \langle\sg,\param^*\rangle 
    = \sum_{i=1}^N |\param^*_i|-\sign(\param_i)\param_i^*
    = \sum_{i=1}^N \param_i^*\left(\sign(\param_i^*) - \sign(\param_i)\right).
\end{align*}
Obviously, the terms in this sum where $\param_i^*=0$ vanish anyways.
Hence, the expression is zero whenever the non-zero entries of $\param$ have the same sign as those of $\param^*$ which is the case if $\norm{\param-\param^*}_\infty<\min\left\lbrace|\param^*_i|\st i\in\{1,\dots,N\},\,\param^*_i\neq 0\right\rbrace$.
Since all norms are equivalent in finite dimensions, one obtains
\begin{align*}
    D_\func^\sg(\param^*,\param) = 0\quad \text{for }\norm{\param-\param^*}\text{ sufficiently small.}
\end{align*}
\revision{Hence, \cref{ass:breg-convex} is locally implied by \cref{ass:mu-convex} if $\func(\param)=\norm{\param}_1$.}
\end{remark}

We would like to remark that---even under the weaker \cref{ass:mu-convex}---one needs some coupling of the loss $\empLoss$ and the regularization functional $\func$ in order to obtain convergence to a critical point.
\citet{benning2018choose} demand that a surrogate function involving both $\empLoss$ and $\func$ admits the KL inequality and that the subgradients of $\func$ \revision{are} bounded close to the minimum $\param^*$ of the loss.
Indeed, in our theory using \cref{ass:mu-convex} it suffices to demand that $\func(\param^*)<\infty$.
This assumption is weaker than assuming bounded subgradients but is nevertheless necessary as the following example taken from \citet{benning2018choose} shows.
\begin{example}[Non-convergence to a critical point]
Let $\empLoss(\param)=(\param+1)^2$ for $\param\in\R$ and $\func(\param)=\chi_{[0,\infty)}(\param)$ be the characteristic function of the positive axis.
Then for any initialization the linearized Bregman iterations \labelcref{eq:lin_breg_it} converge to $\param=0$ which is no critical point of~$\empLoss$.
On the other hand, the only critical point $\param^*=-1$ clearly meets $\func(\param^*)=\infty$.
\end{example}

\begin{theorem}[Convergence in norm]\label{thm:cvgc_norm}
\revision{Assume that \cref{ass:loss,ass:variance,ass:regularizer,ass:mu-convex} hold true and let $\delta>0$.}
Furthermore, assume that the step sizes $\tau^{(k)}$ are such that for all $k\in\N$:
\begin{align*}
    \revision{\tau^{(k)}\leq \frac{\mu}{2\delta L^2}},\qquad
    \tau^{(k+1)} \leq \tau^{(k)}, \qquad
    \sum_{k=0}^\infty (\tau^{(k)})^2 < \infty, \qquad
    \sum_{k=0}^\infty \tau^{(k)} = \infty.
\end{align*}
The function $\mathcal{L}$ has a unique minimizer $\param^*$ and if $J(\param^*)<\infty$ the stochastic linearized Bregman iterations \labelcref{eq:stochlinbreg} satisfy the following:
\begin{itemize}
    \item Letting $d_k:=\E\left[D_{\func_\delta}^{v^{(k)}}(\param^*,\param^{(k)})\right]$ it holds
    \begin{align}\label{ineq:decay_bregman_distance}
        d_{k+1} - d_k + \frac{\mu}{4}\tau^{(k)}\E\left[\norm{\param^*-\param^{(k+1)}}^2\right]
        \leq  \frac{\sigma}{2}\left((\tau^{(k)})^2 +\Exp{\norm{\param^{(k)} - \param^{(k+1)}}^2}\right).
    \end{align}
    \item The iterates possess a subsequence converging in the $L^2$-sense of random variables: 
    \begin{align}
        \lim_{j\to\infty}\E\left[\norm{\param^*-\param^{(k_j)}}^2\right] = 0.
    \end{align}
\end{itemize}
\revision{Here, $\func_\delta$ is defined as in \labelcref{eq:elastic_net}.}
\end{theorem}
\begin{remark}[Choice of step sizes]
A possible step size which satisfies the conditions of \cref{thm:cvgc_norm} is given by $\tau^{(k)} = \tfrac{c}{(k+1)^p}$ where \revision{$0<c<\frac{\mu}{\delta L^2}$} and $p\in(\tfrac{1}{2},1]$.
\end{remark}
\begin{remark}[Deterministic case]
In the deterministic case $\sigma=0$ inequality \labelcref{ineq:decay_bregman_distance} even shows that the Bregman distances decrease along iterations. 
Furthermore, in this case it is not necessary that the step sizes are square-summable and non-increasing since the term on the right hand side does not have to be summed.
\end{remark}
In a finite dimensional setting and using $\ell_1$-regularization one can even show convergence of the whole sequence of Bregman distances.
\revision{
\begin{remark}
With the help of \cref{lem:sum_breg_dist} in the appendix, the quantity $D_{\func_\delta}^{v^{(k)}}(\param^*,\param^{(k)})$ which appears in the decay estimate \labelcref{ineq:decay_bregman_distance} can be simplified as follows:
\begin{align*}
    D_{\func_\delta}^{v^{(k)}}(\param^*,\param^{(k)})
    &=
    \frac{1}{2\delta}\norm{\param^* - \param^{(k)}}^2 + D_\func^{p^{(k)}}(\param^*,\param)
    \\
    &=
    \frac{1}{2\delta}\norm{\param^* - \param^{(k)}}^2 + 
    \func(\param^*) - \func(\param^{(k)}) - \langle \sg^{(k)},\param^*-\param^{(k)}\rangle,
\end{align*}
where $\sg^{(k)}:=v^{(k)}-\frac{1}{\delta}\param^{(k)}\in\partial\func(\param^{(k)})$.
In the case that $\func$ is absolutely $1$-homogeneous, e.g., if $\func(\param)=\norm{\param}_1$ equals the $\ell_1$-norm, this simplifies to
\begin{align*}
    D_{\func_\delta}^{v^{(k)}}(\param^*,\param^{(k)})
    =
    \frac{1}{2\delta}\norm{\param^* - \param^{(k)}}^2 + 
    \func(\param^*) - \langle \sg^{(k)},\param^*\rangle,
\end{align*}
where we used that for absolutely $1$-homogeneous functionals it holds $\langle\sg,\param\rangle=\func(\param)$ for all $\param\in\Param$ and $\sg\in\partial\func(\param)$.
Hence, it measures both the convergence of $\param^{(k)}$ to $\param^*$ in the norm and the convergence of the subgradients $\sg^{(k)}\in\partial\func(\param^{(k)})$ to a subgradient of $\func$ at $\param^*$.
\end{remark}
}
\begin{corollary}[Convergence in finite dimensions]\label[corollary]{cor:convergence_finite_d}
If the parameter space $\Theta$ is finite dimensional and $\func$ equals the $\ell_1$-norm, under the conditions of \cref{thm:cvgc_norm} it even holds $\lim_{k\to\infty} d_k = 0$ which in particular implies $\lim_{k\to\infty}\E\left[\norm{\param^*-\param^{(k)}}^2\right] = 0$.
\end{corollary}

Our second convergence theorem asserts convergence in the Bregman distance and gives quantitative estimates under \cref{ass:breg-convex}, which is a stricter assumption than \cref{ass:mu-convex} and relates the loss function $\empLoss$ with the regularizer; cf. \citet{dragomir2021fast} for a related approach working with $C^2$ functions.
The theorem states that the Bregman distance to the minimizer of the loss can be made arbitrarily small using constant step sizes.
For step sizes which go to zero and are not summable one obtains a quantitative convergence result.

\begin{theorem}[Convergence in the Bregman distance]\label{thm:cvgc_breg_dist}
\revision{Assume that \cref{ass:loss,ass:variance,ass:regularizer,ass:breg-convex} hold true and let $\delta>0$.}
The function $\mathcal{L}$ has a unique \revision{minimizer} $\param^*$ and if $\func(\param^*)<\infty$ the stochastic linearized Bregman iterations \labelcref{eq:stochlinbreg} satisfy the following:
\begin{itemize}
    \item \revision{Letting $d_k :=\E\left[D_{J_\delta}^{v^{(k)}}(\param^*,\param^{(k)})\right]$ it holds
    \begin{align}\label{eq:estimate_expects}
        d_{k+1} \leq \left[1 - \tau^{(k)}\nu\left(1-\tau^{(k)}\frac{2\delta^2 L^2}{\nu}\right)\right]d_k + \delta(\tau^{(k)})^2\sigma^2.
    \end{align}}
    \item For any $\eps>0$ there exists $\tau>0$ such that if $\tau^{(k)}=\tau$ for all $k\in\N$ then
    \begin{align}
        \limsup_{k\to\infty} d_k \leq \eps.
    \end{align}
    \item If $\tau^{(k)}$ is such that
    \begin{align}\label{eq:cond_step_size}
        \lim_{k\to\infty} \tau^{(k)} = 0 \quad \text{and}\quad \sum_{k=0}^\infty \tau^{(k)} = \infty
    \end{align}
    then it holds
    \begin{align}
        \lim_{k\to\infty} d_k = 0.
    \end{align}
\end{itemize}
\revision{Here, $\func_\delta$ is defined as in \labelcref{eq:elastic_net}.}
\end{theorem}
\begin{corollary}[Convergence rate for diminishing step sizes]
\revision{%
The error recursion \labelcref{eq:estimate_expects} coincides with the one for stochastic gradient descent.
In particular, for step sizes of the form $\tau^{(k)}=\tfrac{c}{k}$ with a suitably small constant $c>0$ one can prove with induction \citep{nemirovski2009robust} that $d_k = \frac{C}{k}$ for some $C>0$.}
\end{corollary}

\section{Numerical Experiments}\label{sec:experiments}

In this section we perform an extensive evaluation of our algorithms focusing on different characteristics. 
First, we derive a sparse initialization strategy in \cref{sec:initialization}, using similar statistical arguments as in the seminal works by \citet{bengio10,he2015delving}.
In \cref{sec:comparison,sec:accelerations} we study the influence of hyperparameters and compare our family of algorithms with standard stochastic gradient descent (SGD) and the sparsity promoting proximal gradient descent method.
In \cref{sec:CNNs,sec:resnets} we demonstrate that our algorithms generate sparse and expressive networks for solving the classification task on Fashion-MNIST and CIFAR-10, for which we utilize state-of-the-art CNN and ResNet architectures.
Finally, in \cref{sec:autoencoder} we show that, using row sparsity, our Bregman learning algorithm allows to discover a denoising autoencoder architecture, which shows the potential of the method for architecture design.
Our code is available on \texttt{GitHub} at \url{https://github.com/TimRoith/BregmanLearning} \revision{and relies on \texttt{PyTorch} \cite{NEURIPS2019_9015}}.
%
%
\subsection{Initialization}
\label{sec:initialization}
Parameter initialization for neural networks has a crucial influence on the training process, see \citet{bengio10}. 
In order to tackle the problem of vanishing and exploding gradients, standard methods consider the variance of the weights at initialization \citep{bengio10,he16}, assuming that for each $l$ the entries $W^l_{i,j}$ are i.i.d.~with respect to a probability distribution. 
The intuition here is to preserve the variances over the forward and the backward pass similar to the variance of the respective input of the network, see \citet[Sec.~4.2]{bengio10}.
If the distribution satisfies $\Exp{W^l}=0$, this yields a condition of the form
\begin{align}\label{eq:initcond}
\Var{W^l} = \alpha(n_l,n_{l-1})
\end{align}
where the function $\alpha$ depends on the activation function. 
For anti-symmetric activation functions with $\sigma^{\prime}(0)=1$, as for instance a sigmoidal function, it was shown by \citet{bengio10} that 
\begin{align*}
\alpha(n_l,n_{l-1})=\frac{2}{n_l\cdot n_{l-1}}
\end{align*} 
while for ReLU \citet{he16} suggest to use
\begin{gather*}
\alpha(n_l,n_{l-1})=\frac{2}{n_l}\quad\text{ or }\quad
\alpha(n_l,n_{l-1})=\frac{2}{n_{l-1}}.
\end{gather*}
For our Bregman learning framework we have to adapt this argumentation, taking into account sparsity.
For classical inverse scale space approaches and Bregman iterations of convex losses, as for instance used for image reconstruction \citep{osher2005iterative} and compressed sensing \citep{yin2008bregman,burger2013adaptive,osher2010fast,cai2009linearized}, one would initialize all parameters equal to zero. 
However, for neural networks this yields an unbreakable symmetry of the network parameters, which makes training impossible, see, e.g., \citet[Ch. 6]{Goodfellow16}. 
Instead, we employ an established approach for sparse neural networks (see, e.g., \citet{liu2021,dettmers2019sparse,martens2010deep}) which masks the initial parameters, i.e.,
\begin{align*}
W^l := \tilde W^{l} \odot M^l.
\end{align*}
Here, the mask $M^l\in\R^{n_l,n_{l-1}}$ is chosen 
randomly such that each entry is distributed according to the Bernoulli distribution with a parameter $r\in[0,1]$, i.e.,
\begin{align*}
M^l_{i,j}\sim \mathcal{B}(r).
\end{align*}
The parameter $r$ coincides with the expected percentage of non-zero weights
\begin{align}\label{eq:pnnz}
\mathrm{N}(W^l):=\frac{\norm{W^l}_0}{n_l \cdot n_{l-1}}=
1 - \mathrm{S}(W^l),
\end{align}
where $\mathrm{S}(W^l)$ denotes the sparsity. 
In the following we derive a strategy to initialize $\tilde W^l$.
Choosing $\tilde W^l$ and $M^l$ independent and using $\Exp{\tilde W^l}=0$ standard variance calculus implies
\begin{align*}
\Var{W^l} &= 
\Var{\tilde W^l\odot M^l}\\&= 
\Exp{M^l}^2~\Var{\tilde W^l} + 
\underbrace{\Exp{\tilde W^l}^2~\Var{M^l}}_{=0}+
\Var{M^l}~\Var{\tilde W^l}\\&=
\left(\Exp{M^l}^2 + \Var{M^l}\right)\Var{\tilde W^l}\\&=
\Exp{(M^l)^2}~\Var{\tilde W^l}\\&= 
r~\Var{\tilde W^l}
\end{align*}
and thus deriving from \labelcref{eq:initcond} we obtain the condition
\begin{align}\label{eq:initcond2}
\Var{\tilde W^l} = \frac{1}{r}~\alpha(n_l,n_{l-1}).
\end{align}
Instead of having linear feedforward layers with corresponding weight matrices, the neural network architecture at hand might consist of other groups of parameters which one would like to keep sparse, e.g., using the group sparsity regularization \labelcref{eq:1-2-norm}.
For example, in a convolutional neural network one might be interested in having only a few number of non-zero convolution kernels in order to obtain compact feature representations of the input. 
Similarly, for standard feedforward architectures sparsity of rows of the weight matrices yields compact networks with a small number of active neurons. 
In such cases one can apply the same arguments as above and initialize single elements $\mathbf{g}\in\mathcal{G}$ of a parameter group $\mathcal{G}$ as non-zero with probability $r\in[0,1]$ and with respect to a variance condition similar to \labelcref{eq:initcond2}:
\begin{align}
    \mathbf{g} &= \mathbf{\tilde g} \cdot m ,\quad m\sim\mathcal{B}(r),\\
    \Var{\mathbf{\tilde g}} &= \frac{1}{r} \alpha(\mathbf{g})
\end{align}
Note that these arguments are only valid in a linear regime around the initialization, see \citet{bengio10} for details. 
Hence, if one initializes with sparse parameters but optimizes with very weak regularization, e.g., by using vanilla SGD or choosing $\lambda$ in \labelcref{eq:1-norm} or \labelcref{eq:1-2-norm} very small, the assumption is violated since the sparse initial parameters are rapidly filled during the first training steps. 

The biases are initialized non-sparse and the precise strategy depends on the activation function.
We would like to emphasize that initializing biases with zero is not a 
good idea in the context of sparse neural networks.
In this case, the neuron activations would be equal for all ``dead neurons'' whose incoming weights are zero, which would then yield an unbreakable symmetry.
For ReLU we initialize biases with positive random values to ensure a flow of information and to break symmetry also for dead neurons, which is similar to the strategy proposed by \citet[Ch. 6]{Goodfellow16}.
For other activation functions $\sigma$ which meet $\sigma(x)=0$ if and only if $x=0$, e.g., TanH, Sigmoid, or Leaky ReLU, biases can be initialized with random numbers of arbitrary sign.

%
%
\subsection{Comparison of Algorithms}\label{sec:comparison}
\begin{figure}[htb]
\centering
\includegraphics[width=\textwidth,
                trim=0.2cm 0.0cm 0.2cm 0.0cm,clip]{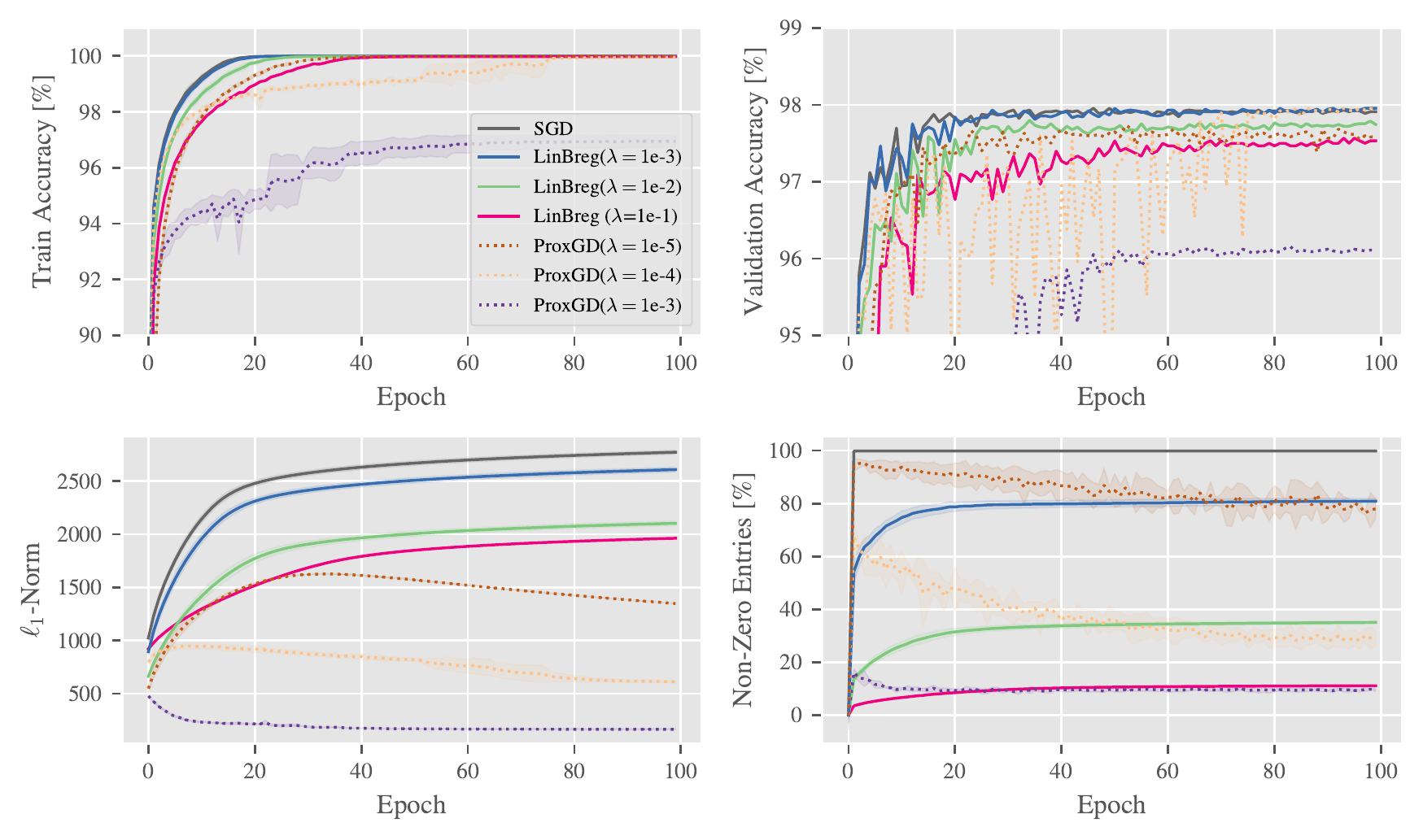}%
\caption{Comparison of vanilla SGD (black solid line), \LinBreg{} (colored solid lines), and ProxGD (colored dotted lines) for different regularization parameters on MNIST.
The curves show the averaged accuracies on train and validation sets, $\ell_1$-norms, and non-zero entries over three runs. The shaded area visualizes the standard deviation.}
\label{fig:fc_net}
\end{figure}
We start by comparing the proposed \LinBreg{} \cref{alg:proximal_bregman_training} with vanilla stochastic gradient descent (SGD) without sparsity regularization and with the Lasso-based approach from \citet{scardapane2017group}, for which we compute solutions to the sparsity-regularized risk minimization problem \labelcref{eq:reg_emp_risk} using the proximal gradient descent algorithm (ProxGD) from \labelcref{eq:proxGD}.

We consider the classification task on the MNIST dataset \citep{leCun10} for studying the impact of the hyperparameters of these methods. 
The set consists of $60,000$ images of handwritten digits which we split into $55,000$ images used for the training and $5,000$ images used for a validation process during training.
We train a fully connected net with \revision{ReLU activations and} two hidden layers ($200$ and $80$ neurons), and use the $\ell_1$-regularization from \labelcref{eq:1_norm_ffwd},
\begin{align*}
J(\param) = \lambda\sum_{l=1}^L \norm{W^l}_{1,1}
\end{align*}
In \cref{fig:fc_net} we compare the training results of vanilla SGD, the proposed \LinBreg{}, and the ProxGD algorithm.
Following the strategy introduced in \cref{sec:initialization} we initialize the weights with $1\%$ non-zero entries, i.e., $r=0.01$.
The learning rate is chosen as $\tau=0.1$ and is multiplied by a factor of $0.5$ whenever the validation accuracy stagnates. 
For a fair comparison the training is executed for three different fixed random seeds, and the plot visualizes mean and standard deviation of the three runs, respectively. 
We show the training and validation accuracies, the $\ell_1$-norm, and the overall percentage of non-zero weights.
Note that for the validation accuracies we do not show standard deviations for the sake of clarity.

While SGD without sparsity regularization instantaneously destroys sparsity, \LinBreg{} exhibits the expected inverse scale space behaviour, where the number of non-zero weights gradually grows during training, and the train accuracy increases monotonously.
\revision{This is suggested by \cref{thm:decreasing_loss}, even though our experimental setup, in particular the non-smooth ReLU activation functions, is not covered by the theoretical framework which require at least $L$-smoothness of the loss functions.}

In contrast, ProxGD shows no monotonicity of training accuracy or sparsity and the validation accuracies oscillate heavily.
Instead, it adds a lot of non-zero weights in the beginning and then gradually reduces them.
Obviously, the regularized empirical risk minimization \labelcref{eq:reg_emp_risk} implies a trade-off between training accuracy and sparsity which depends on $\lambda$.
For \LinBreg{} this trade-off is neither predicted by theory nor observed numerically. 
Here, the regularization parameter only induces a trade-off between \emph{validation} accuracy and sparsity, which is to be expected.

\LinBreg{} (blue curves) can generate networks whose validation accuracy equals the one of a full network and use only $80\%$ of the weights.
For the largest regularization parameter (magenta curves) \LinBreg{} uses only $10\%$ of the weights and still does not drop more than half a percentage point in validation accuracy.
%
%
\subsection{Accelerated Bregman Algorithms}\label{sec:accelerations}

\begin{figure}[htb]
\centering
\includegraphics[width=\textwidth]{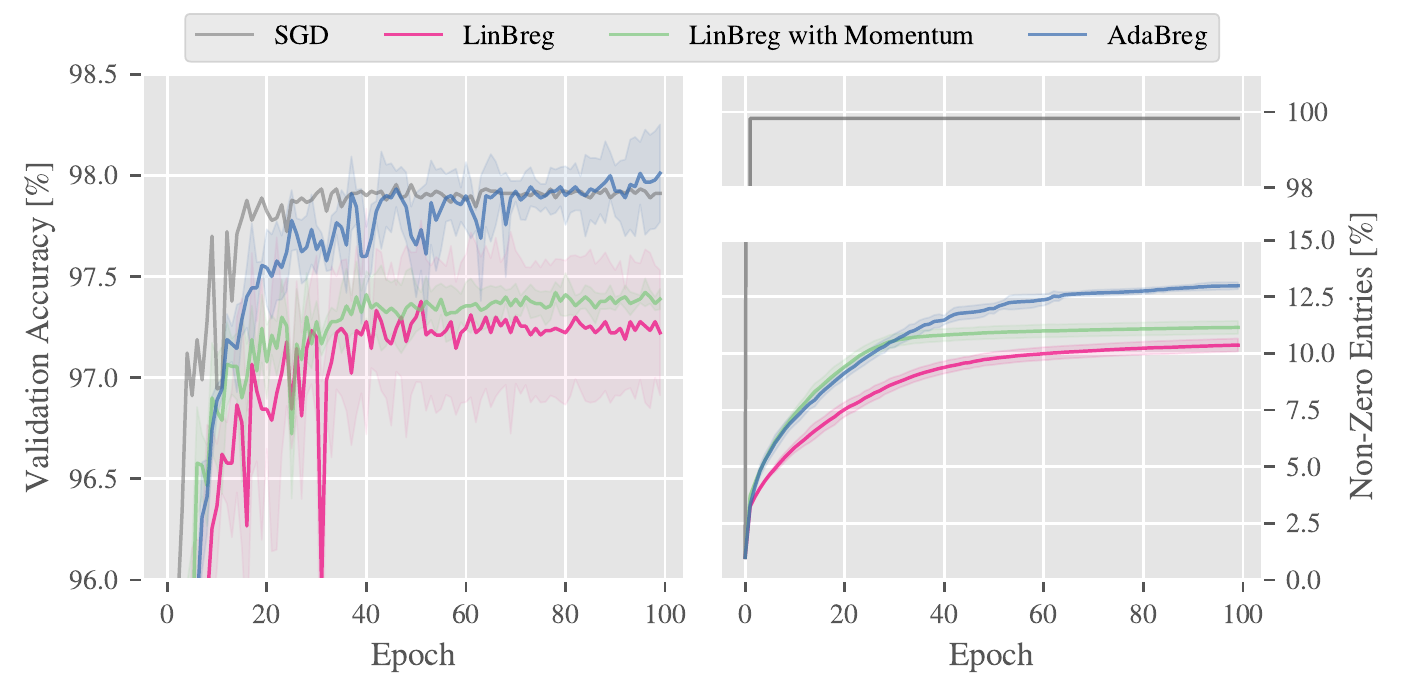}
\caption{Comparison of \LinBreg{} (with momentum), \AdaBreg{}, and vanilla SGD. The networks generated by \AdaBreg{} are sparse and generalize better.}
\label{fig:momentum_comparison}
\end{figure}

We use the same setup as in the previous section to compare \LinBreg{} with its momentum-based acceleration and \AdaBreg{} (see \cref{alg:proximal_bregman_training,alg:proximal_bregman_training_momentum,alg:proximal_bregman_training_adam}).
Using the regularization parameter $\lambda=10^{-1}$ from the previous section (see the magenta curves), \cref{fig:momentum_comparison} shows \revision{the validation accuracy of the different networks trained with} \LinBreg{}, \LinBreg{} with momentum, and \AdaBreg{}.
For comparison we visualize again the results of vanilla SGD as gray curve.
It is obvious that all three proposed algorithms generate very accurate networks using approximately $10\%$ of the weights. 
As expected, the accelerated versions increase both the validation accuracy and the number of non-zero parameters faster than the baseline algorithm \LinBreg{}.
While after 100 epochs \LinBreg{} and its momentum version have slightly lower validation accuracies than the non-sparse networks generated by SGD, \AdaBreg{} outperforms the other algorithms including SGD in terms of validation accuracy while maintaining a high degree of sparsity.
%
%
\subsection{Sparsity for Convolutional Neural Networks (CNNs)}\label{sec:CNNs}

In this example we apply our algorithms to a convolutional neural network of the form 
\begin{align*}
\text{$5\times 5$ conv, 64} \overset{\text{Maxpool$/2$}}{\longrightarrow} 
\text{$5\times 5$ conv, 64} \overset{\text{Maxpool$/2$}}{\longrightarrow}
\text{fc 1024} \longrightarrow 
\text{fc 128}\longrightarrow
\text{fc 10}
\end{align*}
with ReLU activations to solve the classification task on Fashion-MNIST.
We run experiments both for sparsity regularization utilizing the $\ell_1$-norm \labelcref{eq:1-norm} and for a combination of the $\ell_1$-norm on the linear layers and the group $\ell_{1,2}$-norm \labelcref{eq:1-2-norm} on the convolutional kernels.
This way, we aim to obtain compressed network architectures with only few active convolutional filters.

The inverse scale space character of our algorithms is visualized in \cref{fig:kernels} from the beginning of the paper which shows the 64 feature maps of an input image, generated by the first convolutional layer of the network after $0$, $5$, $20$, and $100$ epochs of \LinBreg{} with group sparsity.
One can observe that gradually more kernels are added until iteration $20$, from where on the number of kernels stays fixed and the kernels themselves are optimized.

\cref{tab:F-MNIST_simple_sparsity,tab:F-MNIST_group_sparsity} shows the test and training accuracies as well as sparsity levels.
For plain sparsity regularization we only show the total sparsity level of all network parameters whereas for the group sparsity regularization we show the sparsity of the linear layers and the relative number of non-zero convolutional kernels.

A convolutional layer for a input $z\in\R^{c_{l-1},n_{l-1},m_{l-1}}$ is given as 
\begin{align*}
    \Phi^l_j(z) = b_j + \sum_{i=1}^{c_{l-1}} K^{l}_{i,j}\ast z_{i,\bullet},
\end{align*}
where $K^l_{i,j}\in\R^{k,k}$ denote kernel matrices with corresponding biases $b_j\in\R^{n_l,m_l}$ for in-channels $i\in\{1,\ldots,c_{l-1}\}$ and out-channels $j\in\{1,\ldots,c_{l}\}$. 
Therefore, we denote by
\begin{align*}
    \mathrm{N}_{\mathrm{conv}}:=\frac{\sum_{l\in I_{\mathrm{conv}}}
    \#\{K^l_{i,j}: K^l_{i,j}\neq \mathbf{0}\}}
    {\sum_{l\in I_{\mathrm{conv}}} c_l\cdot c_{l-1}}
\end{align*}
the percentage of non-zero kernels of the whole net where $I_{\mathrm{conv}}$ denotes the index set of the convolutional layers. 
Analogously, using a similar term as in \labelcref{eq:pnnz} we denote by 
\begin{align*}
    \mathrm{N}_{\mathrm{linear}}:=\frac{\sum_{l\in I_{\mathrm{linear}}}
    \norm{W^l}_0}
    {\sum_{l\in I_{\mathrm{linear}}} n_l\cdot n_{l-1}}
\end{align*}
the percentage of weights used in the linear layers.
Finally, we define $\mathrm{N}_\mathrm{total}:=\mathrm{N}_\mathrm{conv} + \mathrm{N}_\mathrm{linear}$.

We compare our algorithms \LinBreg{} (with momentum) and \AdaBreg{} against vanilla training without sparsity, iterative pruning \citep{han2015learning}, and the Group Lasso approach from \citet{scardapane2017group}, and train all networks to a comparable sparsity level, given in brackets.
The pruning scheme is taken from \citet{han2015learning}, where in each step a certain amount of weights is pruned, followed by a retraining step. 
For our experiment the amount of weights pruned in each iteration was chosen, so that a specified target sparsity is met.
For the Group Lasso approach, which is based on the regularized risk minimization \labelcref{eq:reg_emp_risk}, we use two different optimizers.
First, we apply SGD applied to the \labelcref{eq:reg_emp_risk} and apply thresholding afterwards to obtain sparse weights, which is the standard approach in the community (cf.~\citet{scardapane2017group}).
Second, we apply proximal gradient descent~\labelcref{eq:proxGD} to \labelcref{eq:reg_emp_risk} which yields \revision{sparse} solutions without need for thresholding.
Our Bregman algorithms were initialized with $1\%$ non-zero parameters, following the strategy from \cref{sec:initialization}, all other algorithms were initialized non-sparse using standard techniques \citep{bengio10,he2015delving}.   
\revision{For all algorithms we tuned the hyperparameters (e.g., pruning rate, regularization parameter for Group Lasso and Bregman) in order to achieve comparable sparsity levels.

Note that it is non-trivial to compare different algorithms for sparse training since they optimize different objectives, and both the sparsity, the train, and the test accuracy of the resulting networks matter.
Therefore, we show results whose sparsity levels and accuracies are in similar ranges.
}

\cref{tab:F-MNIST_simple_sparsity} shows that all algorithms manage to compute very sparse networks with ca. $2\%$ drop in test accuracy on Fasion-MNIST, compared to vanilla dense training with Adam.
Note that we optimized the hyperparameters (regularization and thresholding parameters) of all algorithms for optimal performance on a validation set, subject to having comparable sparsity levels.
Our algorithms \LinBreg{} and \AdaBreg{} yield sparser networks with the same accuracies as Pruning and Lasso.

Similar observations are true for \cref{tab:F-MNIST_group_sparsity} where we used group sparsity regularization on the convolutional kernels.
Here all algorithms apart from pruning yield similar results, whereas pruning exhibits a significantly worse test accuracy despite using a larger number of non-zero parameters. 
\revision{The combination of SGD-optimized Group Lasso with subsequent thresholding yields the best test accuracy using a moderate sparsity level.}

As mentioned above the Lasso and Group Lasso results using SGD underwent an additional thresholding step after training in order to generate sparse solutions. 
Obviously, one could also do this with the results of ProxGD and Bregman which would further improve their sparsity levels.
However, in this experiment we refrain from doing so in order not to change the nature of the algorithms.
\begin{table}[htb]
\small%
\noindent%
\begin{tabularx}{\textwidth}{|R R||C|c c|}
Strategy & Optimizer & 
$\mathrm{N}_{\mathrm{total}}$ in [\%] & 
Test Acc& Train Acc\\
\hhline{|=====|}
Vanilla & Adam   &100&92.1&100.0\\
\hhline{-----}
\multirow{1}{*}{Pruning ($5\%$)}
        &SGD&4.7&89.2&92.0\\
\hhline{-----}
\multirow{2}{*}{Lasso}
        &SGD 
        + thresh. &3.5&90.1&94.7\\
        &ProxGD 
        &4.8&89.4&91.4\\
\hhline{-----}
\multirow{3}{*}{\bf Bregman}
                &\LinBreg{} 
                &1.9&89.2&91.1\\
                &\LinBreg{} ($\beta=0.9$) 
                &2.7&89.9&93.8\\
                %
                %
                %
                &\AdaBreg{}
                &2.3&{90.5}&93.6
                %
                %
                %
\end{tabularx}
\caption{Sparsity levels and accuracies on the Fashion-MNIST data set.}\label{tab:F-MNIST_simple_sparsity}
\end{table}
\begin{table}[htb]
\small%
\centering%
\begin{tabularx}{\textwidth}{|p{1.8cm} p{2.8cm}||c c|c c|}
Strategy & Optimizer & 
$\mathrm{N}_{\mathrm{linear}}$ in [\%] & 
$\mathrm{N}_{\mathrm{conv}}$ in [\%]&  Test Acc &Train Acc\\
\hhline{|======|}
Vanilla & Adam   &100&100&92.1&100.0\\
\hhline{------}
\multirow{1}{*}{Pruning ($7\%$)}
                        & SGD &7.0&6.5&86.9&89.9\\
\hhline{------}
\multirow{2}{*}{GLasso}
        &SGD + thresh. 
        &3.6&4.3&90.3&94.8\\
        &ProxGD 
        &3.0&3.7&89.8&91.6\\
\hhline{------}
\multirow{3}{*}{\shortstack{\bf Bregman}}
                        &\LinBreg{} 
                        &3.8 &4.2 &89.5&93.1\\
                        &\LinBreg{} ($\beta=0.9$) 
                        &3.5 &4.7 &89.9&93.5\\
                        &AdaBreg 
                        &3.5 &2.8 &89.4&92.6
\end{tabularx}
\caption{Group sparsity levels and accuracies on the Fashion-MNIST data set.}\label{tab:F-MNIST_group_sparsity}
\end{table}


\subsection{Residual Neural Networks (ResNets)}\label{sec:resnets}

In this experiment we trained a ResNet-18 architecture for classification on CIFAR-10, enforcing sparsity through the $\ell_1$-norm \labelcref{eq:1-norm} and comparing different strategies, as before.
\cref{tab:CIFAR} shows the resulting sparsity levels of the total number of parameters and the percentage of non-zero convolutional kernels as well as the train and test accuracies.
Note that even though we used the standard $\ell_1$ regularization \labelcref{eq:1-norm} and no group sparsity, the trained networks exhibit large percentages of zero-kernels.

For comparison we also show the unregularized vanilla results using SGD with momentum and Adam, which both use $100\%$ of the parameters.
The \LinBreg{} result with thresholding shows that one can train a very sparse network using only $3.4\%$ of all parameters with $3.4\%$ drop in test accuracy.
With \AdaBreg{} we obtain a sparsity level of $14.7\%$, resulting in a drop of only $1.3\%$.
\revision{The combination of Adam-optimized Lasso with subsequent thresholding yields a $3\%$ sparsity with a drop of $3.6\%$ in test accuracy, which is the sparsest result in this comparison.}

\begin{table}[htb]
\small
\begin{tabularx}{\textwidth}{|p{1.5cm} p{3.3cm}||c c|c c|}
Strategy & Optimizer & 
$\mathrm{N}_{\mathrm{total}}$ in [\%] & 
$\mathrm{N}_{\mathrm{conv}}$ in [\%] & 
Test Acc&Train Acc\\
\hhline{|======|}
\multirow{2}{*}{Vanilla}
        &SGD with momentum &100.0&100.0&92.15&99.8\%\\
        &Adam &100.0&100.0&93.6&100.0\%\\
\hhline{------}
\multirow{2}{*}{Lasso}
        &Adam &99.7&100.0&91.1 &100\\
        &Adam + thresh.&3.0&15.7&90.0 &99.8\\
\hhline{------}
\multirow{6}{*}{\bf{Bregman}}
                &\LinBreg{}
                &5.5&24.8&90.9&99.5\\
                &\LinBreg{} + thresh. 
                &3.4&16.9&90.2&99.4\\
                &\LinBreg{} ($\beta=0.9$)
                &4.8&21.0&90.4&100.0\\
                &\LinBreg{} ($\beta=0.9$) + thresh. 
                &3.6&17.4&90.0&99.9\\
                &\AdaBreg{} 
                &14.7&56.7&92.3&100.0\\
                &\AdaBreg{} + thresh. 
                &9.2&42.2&90.5&99.9\\
\end{tabularx}
\caption{Sparsity levels and accuracies on the CIFAR-10 data set.}\label{tab:CIFAR}
\end{table}

\subsection{Towards Architecture Design: Unveiling an Autoencoder}\label{sec:autoencoder}

In this final experiment we investigate the potential of our Bregman training framework for architecture design, which refers to letting the network learn its own architecture.
\citet{hoefler2021sparsity} identified this as one of the main potentials of sparse training.

The inverse scale space character of our approach turns out to be promising for starting with very sparse networks and letting the network choose its own architecture by enforcing, e.g., row sparsity of weight matrices.
The simplest yet widely used non-trivial architecture one might hope to train is an autoencoder, as used, e.g., for denoising images.
To this end, we utilize the MNIST data set and train a fully connected feedforward network with five hidden layers, all having the same dimension as the input layer, to denoise MNIST images. \revision{Similar to the experiments in \cref{sec:comparison} we split the dataset in 55,000 images used for training and 5,000 images to evaluate the performance.} 
We enforce row sparsity by using the regularizer \labelcref{eq:1-2_norm_ffwd}, which in this context is equivalent to having few active neurons in the network.

\cref{fig:encoder} shows the number of active neurons in the network at different stages of the training process using \LinBreg{}.
Here darker colors indicate more iterations.
We initialized around $1\%$ of all rows non-zero, which corresponds to 8 neurons per layer being initially active.
Our algorithm successively adds neurons and converges to an autoencoder-like structure, where the number \revision{of neurons} decreases until the middle layer and then increases again.
Note the network developed this structure ``on its own'' and that we \emph{did not} artificially generate this result by using different regularization strengths for the different layers.
In \cref{fig:denoised_images} we additionally show the denoising performance of the trained network on some images from the MNIST test set. \revision{The network was trained for 100 iterations, using a regularization value of $\lambda=0.07$ in \labelcref{eq:1-2_norm_ffwd} and employing a standard MSE loss. We evaluate the test performance using the established structural similarity index measure (SSIM) \citep{Wang04imagequality}, which assigns values close to 1 to pairs of images that are perceptually similar. Averaging this value over the whole test set we report a value of $\text{SSIM}\approx 0.93$. For comparison, we also trained a network with 100 iterations of standard SGD which yields no sparsity and a value of $\text{SSIM} \approx 0.89$.
}

\begin{figure}[htp]
\centering
\includegraphics[width=\textwidth,
                 trim=0.6cm 0.2cm 0.2cm 0.3cm,clip]{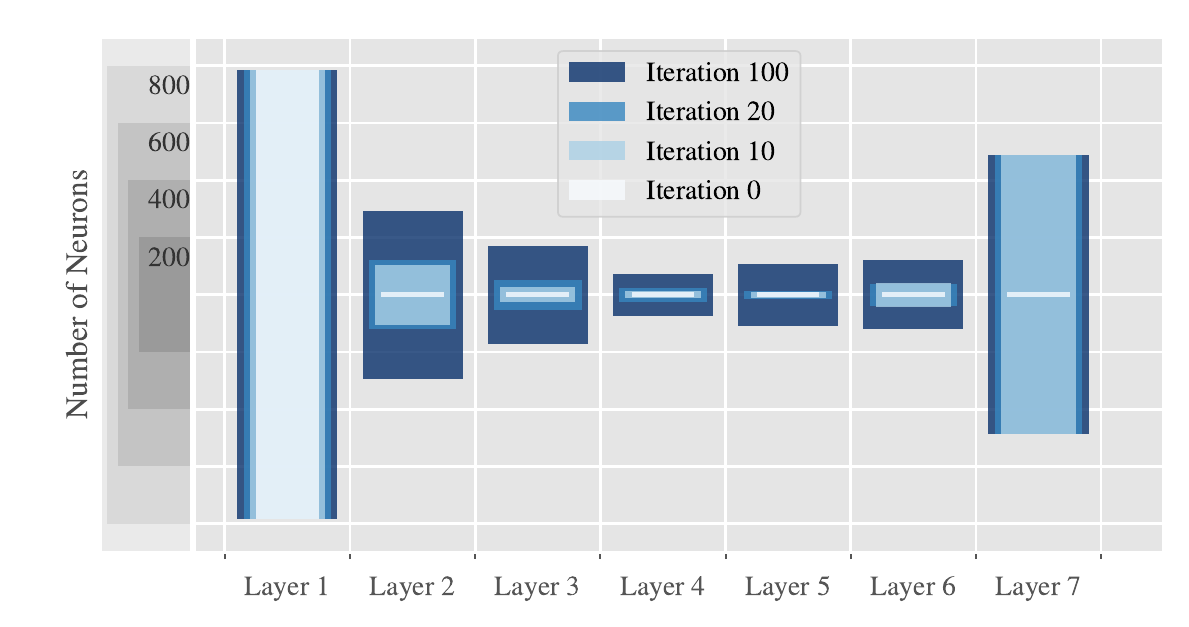}
\caption{Architecture design for denoising: \LinBreg{} automatically unveils an autoencoder.}
\label{fig:encoder}
\end{figure}

\begin{figure}[htb]
\centering
\includegraphics[width=\textwidth,
                 trim=0.3cm 0.3cm 0.2cm 0.3cm,clip]{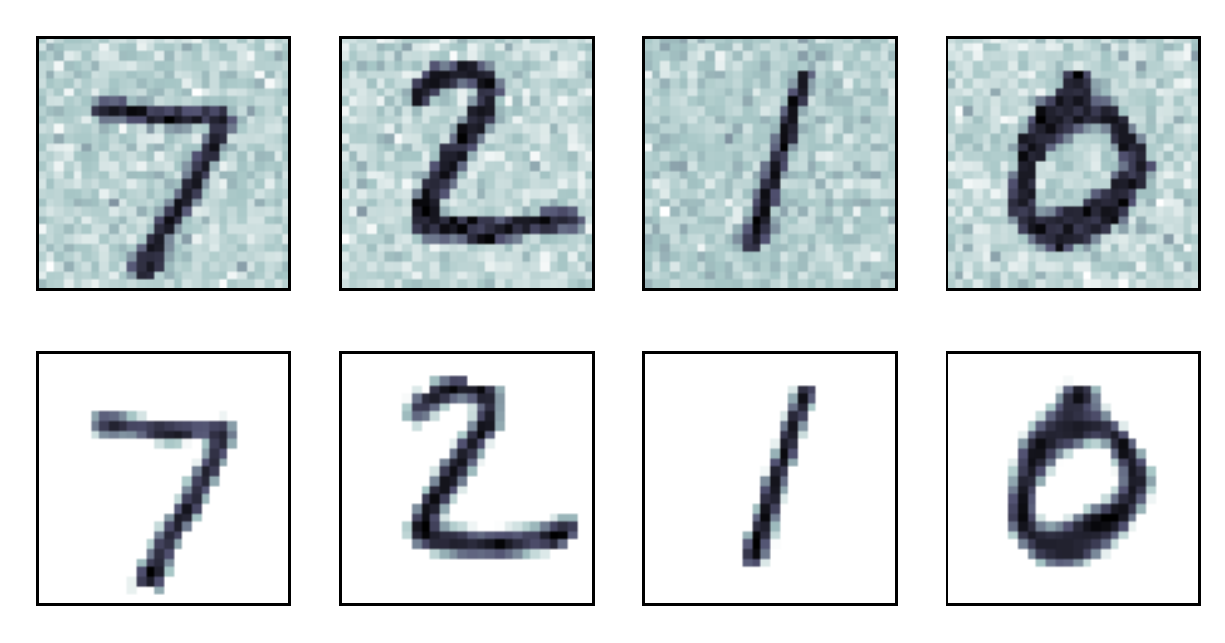}
\caption{The denoising performance of the trained autoencoder on the test set \revision{with an average SSIM value of $\approx0.93$.}}
\label{fig:denoised_images}
\end{figure}

\section{Conclusion}\label{sec:conclusion}

In this paper we proposed an inverse scale space approach for training sparse neural networks based on linearized Bregman iterations.
We introduced \emph{LinBreg} as baseline algorithm for our learning framework and also discuss two variants using momentum and Adam.
The effect of incorporating Bregman iterations into neural network training was investigated in numerical experiments on benchmark data sets.
Our observations showed that the proposed method is able to train very sparse and accurate neural networks in an inverse scale space manner without using additional heuristics. 
Furthermore, we gave a glimpse of its applicability for discovering suitable network architectures for a given application task, e.g., an autoencoder architecture for image denoising.
We mathematically supported our findings by performing a stochastic convergence analysis of the loss decay, and we proved convergence of the parameters in the case of convexity.

The proposed Bregman learning framework has a lot of potential for training sparse neural networks, and there are still a few open research questions (see also \citet{hoefler2021sparsity}) which we would like to emphasize in the following.

First, we would like to use the inverse scale space character of the proposed Bregman learning algorithms in combination with sparse backpropagation for resource-friendly training, hence improving the carbon footprint of training~\citep{anthony2020carbontracker}.
This is \revision{a} non-trivial endeavour for the following reason: A-priori it is not clear which weights are worth updating since estimating the magnitude of the gradient with respect to these weights already requires evaluating the backpropagation.
A possible way to achieve this consists in performing a Bregman step to obtain a sparse support of the weights, performing several masked backpropagation steps to optimize the weights in these positions, and alternate this procedure.

Second, our experiment from \cref{sec:autoencoder}, where our algorithm discovered a denoising autoencoder, suggests that our method has great potential for general architecture design tasks. 
Using suitable sparsity regularization, e.g., on residual connections and rows of the weight matrices, one can investigate whether networks learn to form a U-net \citep{ronneberger2015u} structure for the solution of inverse problems.

\revision{On the analysis side}, it is worth investigating the convergence of \LinBreg{} in the fully non-convex setting based on the Kurdyka-\L ojasiewicz inequality and to extend these results to our accelerated algorithms \LinBreg{} with momentum and \AdaBreg{}.
\revision{Furthermore, it will be interesting to remove the bounded variance condition from \cref{ass:variance}, which is known to be possible for stochastic gradient descent if the batch losses satisfy \labelcref{ineq:L-smooth}, see \citet{lei2019stochastic}.
Finally, the characterizing the limit point of \LinBreg{} for non-strongly convex losses as minimizer which minimizes the Bregman distance to the initialization will be worthwhile.
}


\acks{This work was supported by the European Union's Horizon 2020 research and innovation programme under the Marie Sk\l odowska-Curie grant agreement No.~777826 (NoMADS) and by the German Ministry of Science and Technology (BMBF) under grant agreement No.~05M2020 (DELETO).
Additionally we thank for the financial support by the Cluster of Excellence ‘‘Engineering of Advanced Materials’’ (EAM) and the "Competence Unit for Scientific Computing" (CSC) at the University of Erlangen-Nürnberg (FAU).
LB acknowledges support by the Deutsche Forschungsgemeinschaft (DFG, German Research Foundation) under Germany's Excellence Strategy - GZ 2047/1, Projekt-ID 390685813.
}

\bibliography{bibliography}

\appendix

\section*{Appendix}

In all proofs we will use the abbreviation $g^{(k)}:=g(\param^{(k)};\omega^{(k)})$ as in \labelcref{eq:stochlinbreg}.

\section{Proofs from \texorpdfstring{\cref{sec:loss_decay}}{Section 3.1}}

\begin{proof}[Proof of \cref{thm:decreasing_loss}]
Using \labelcref{ineq:L-smooth} one obtains
\begin{align*}
    &\quad\empLoss(\param^{(k+1)}) - \empLoss(\param^{(k)}) \\
    &\leq  \langle\nabla\empLoss(\param^{(k)}),\param^{(k+1)}-\param^{(k)}\rangle + \frac{L}{2}\norm{\param^{(k+1)}-\param^{(k)}}^2 \\
    &= \langle g^{(k)},\param^{(k+1)}-\param^{(k)}) + \langle\nabla\empLoss(\param^{(k)}) - g^{(k)},\param^{(k+1)}-\param^{(k)}\rangle +\frac{L}{2}\norm{\param^{(k+1)}-\param^{(k)}}^2 \\
    &\revision{=} -\frac{1}{\tau^{(k)}} \langle v^{(k+1)}-v^{(k)} , \param^{(k+1)}-\param^{(k)}\rangle \\
    &\qquad+ \langle\nabla\empLoss(\param^{(k)}) - g^{(k)},\param^{(k+1)}-\param^{(k)}\rangle
    + \frac{L}{2}\norm{\param^{(k+1)}-\param^{(k)}}^2 \\
    &= - \frac{1}{\tau^{(k)}} D_\func^{\mathrm{sym}}(\param^{(k+1)},\param^{(k)})-\frac{1}{\delta\tau^{(k)}}\norm{\param^{(k+1)}-\param^{(k)}}^2 \\
    &\qquad + \langle\nabla\empLoss(\param^{(k)}) - g^{(k)},\param^{(k+1)}-\param^{(k)}\rangle + \frac{L}{2}\norm{\param^{(k+1)}-\param^{(k)}}^2.
\end{align*}
Reordering and using \revision{the Cauchy-Schwarz} inequality yields
\begin{align*}
    \empLoss(\param^{(k+1)}) - \empLoss(\param^{(k)}) + \frac{1}{\tau^{(k)}} D_\func^{\mathrm{sym}}(\param^{(k+1)},\param^{(k)}) + &\frac{2 - L\delta\tau^{(k)}}{2\delta\tau^{(k)}}\norm{\param^{(k+1)}-\param^{(k)}}^2 \leq\\
    &\norm{\nabla\empLoss(\param^{(k)})-g^{(k)}}\norm{\param^{(k+1)}-\param^{(k)}}. 
\end{align*}
Taking expectations and using Young's inequality gives for any $c>0$
\begin{align*}
    \Exp{\empLoss(\param^{(k+1)})} &- \Exp{\empLoss(\param^{(k)})} + \frac{1}{\tau^{(k)}} \Exp{D_\func^{\mathrm{sym}}(\param^{(k+1)},\param^{(k)})} \\
    &+ \frac{2 - L\delta\tau^{(k)}}{2\delta\tau^{(k)}}\Exp{\norm{\param^{(k+1)}-\param^{(k)}}^2} \leq
    \tau^{(k)}\delta\frac{\sigma^2}{2c}+ \frac{c}{{2\delta\tau^{(k)}}}\Exp{\norm{\param^{(k+1)}-\param^{(k)}}^2}.
\end{align*}
If $c$ is sufficiently small and $\tau^{(k)}<\tfrac{2}{L\delta}$, we can absorb the last term into the left hand side and obtain
\begin{align*}
    \Exp{\empLoss(\param^{(k+1)})} - \Exp{\empLoss(\param^{(k)})} + \frac{1}{\tau^{(k)}} \Exp{D_\func^{\mathrm{sym}}(\param^{(k+1)},\param^{(k)})} + &\frac{C}{2\delta\tau^{(k)}}\Exp{\norm{\param^{(k+1)}-\param^{(k)}}^2} \\
    &\leq \tau^{(k)}\delta\frac{\sigma^2}{2c},
\end{align*}
where $C>0$ is a suitable constant.
This shows \labelcref{ineq:loss_decay}.
\end{proof}

\begin{proof}[Proof of \cref{cor:square_sum}]
Using the assumptions on $\tau^{(k)}$, we can multiply \labelcref{ineq:loss_decay} with $\tau^{(k)}$ and sum up the resulting inequality to obtain
\begin{align*}
    &\tau^{(K)}\Exp{\empLoss(\param^{(K)})} - \tau^{(0)}\Exp{\empLoss(\param^{(0)})} +\\
    &\quad\sum_{k=0}^{K-1}\left(\Exp{D_\func^{\mathrm{sym}}(\param^{(k+1)},\param^{(k)})} +
    \frac{C}{2\delta}\Exp{\norm{\param^{(k+1)}-\param^{(k)}}^2}\right)
    \leq  \delta\frac{\sigma^2}{2c}\sum_{k=0}^{K-1}(\tau^{(k)})^2
\end{align*}
Since $\empLoss\geq 0$ we can drop the first term. 
Sending $K\to\infty$ and using that the step sizes are square-summable concludes the proof.
\end{proof}

\section{Proof from \texorpdfstring{\cref{sec:cvgc_iterates}}{Section 3.2}}

The following lemma allows to split the Bregman distance with respect to the elastic net functional $\func_\delta$ into two Bregman distances with respect to the functional $\func$ and the Euclidean norm.
\begin{lemma}\label[lemma]{lem:sum_breg_dist}
For all $\tilde\param,\param\in\Param$ and $v\in\partial \func_\delta(\param)$ it holds that
\begin{align*}
    D_{\func_\delta}^v(\tilde\param,\param) = D_\func^\sg(\tilde\param,\param) + \frac{1}{2\delta}\norm{\tilde\param-\param}^2,
\end{align*}
\revision{where $\sg:= v - \frac{1}{\delta}\param \in \partial\func(\param)$.}
\end{lemma}
\begin{proof}
Since $\partial \func_\delta(\param) = \partial \func(\param) + \partial\frac{1}{2\delta}\norm{\param}^2$, we can write $v=\sg+\frac{1}{\delta}\param$ with $\sg\in\partial \func(\param)$.
This readily yields
\begin{align*}
    D_{\func_\delta}^v(\tilde\param,\param) 
    &= \func_\delta(\tilde\param) - \func_\delta(\param) - \langle v,\tilde\param-\param\rangle\\ 
    &= \func(\tilde\param) + \frac{1}{2\delta}\norm{\tilde\param}^2 - \func(\param) - \frac{1}{2\delta}\norm{\param}^2 - \langle \sg, \tilde\param-\param\rangle - \frac{1}{\delta}\langle\param,\tilde\param-\param\rangle\\
    &= D^\sg_\func(\tilde\param,\param) + \frac{1}{2\delta}\norm{\tilde\param}^2 - \frac{1}{2\delta}\norm{\param}^2 - \frac{1}{\delta}\langle\param,\tilde\param\rangle + \frac{1}{\delta}\norm{\param}^2\\
    &= D_\func^\sg(\tilde\param,\param) + \frac{1}{2\delta}\norm{\tilde\param-\param}^2.
\end{align*}
\end{proof}

The next lemma expresses the difference of two subsequent Bregman distances along the iteration in two different ways. 
The first one is useful for proving convergence of \labelcref{eq:stochlinbreg} under the weaker convexity \cref{ass:mu-convex}, whereas the second one is used for proving the stronger convergence statement \cref{thm:cvgc_breg_dist} under \cref{ass:breg-convex}.

\begin{lemma}\label[lemma]{lem:iterates}
Denoting $d_k:=\E\left[D_{\func_\delta}^{v^{(k)}}(\param^*,\param^{(k)})\right]$ the iteration \labelcref{eq:stochlinbreg} fulfills:
\begin{align}
    \label{eq:breg_diff_1}
    d_{k+1} - d_k &= 
    -\E\left[D_{\func_\delta}^{v^{(k)}}(\param^{(k+1)},\param^{(k)})\right] + \tau^{(k)}\E\left[\langle g^{(k)},\param^*-\param^{(k+1)}\rangle\right], \\
    \label{eq:breg_diff_2}
    d_{k+1} - d_k &= 
    \E\left[D_{\func_\delta}^{v^{(k+1)}}(\param^{(k)},\param^{(k+1)})\right] + \tau^{(k)}\E\left[\langle\nabla \mathcal L(\param^{(k)}),\param^*-\param^{(k)}\rangle\right].
\end{align}
\end{lemma}
\begin{proof}
We compute using the update in \eqref{eq:stochlinbreg}
\begin{align*}
    &\phantom{=} D^{v^{(k+1)}}_{\func_\delta}(\param^*,\param^{(k+1)})-D^{v^{(k)}}_{\func_\delta}(\param^*,\param^{(k)})\\
    &= \func_\delta(\param^{(k)})-\func_\delta(\param^{(k+1)}) - \langle v^{(k+1)},\param^*-\param^{(k+1)}\rangle + \langle v^{(k)},\param^*-\param^{(k)}\rangle \\
    &= -\left(\func_\delta(\param^{(k+1)}) - \func_\delta(\param^{(k)}) - \langle v^{(k)}, \param^{(k+1)} - \param^{(k)} \rangle\right) - \langle v^{(k)}, \param^{(k+1)} -\param^{(k)} \rangle \\
    &\qquad\qquad - \langle v^{(k+1)},\param^*-\param^{(k+1)}\rangle + \langle v^{(k)},\param^*-\param^{(k)}\rangle \\
    &= -D_{\func_\delta}^{v^{(k)}}(\param^{(k+1)},\param^{(k)}) + \langle v^{(k)}-v^{(k+1)}, \param^* - \param^{(k+1)} \rangle \\
    &= -D_{\func_\delta}^{v^{(k)}}(\param^{(k+1)},\param^{(k)}) + \tau^{(k)} \langle g^{(k)}, \param^* - \param^{(k+1)} \rangle.
\end{align*}
For the second equation we similarly compute 
\begin{align*}
    &\phantom{=} D^{v^{(k+1)}}_{\func_\delta}(\param^*,\param^{(k+1)})-D^{v^{(k)}}_{\func_\delta}(\param^*,\param^{(k)})\\
    &= \func_\delta(\param^{(k)})-\func_\delta(\param^{(k+1)}) - \langle v^{(k+1)},\param^*-\param^{(k+1)}\rangle + \langle v^{(k)},\param^*-\param^{(k)}\rangle \\
    &= \func_\delta(\param^{(k)})-\func_\delta(\param^{(k+1)}) - \langle v^{(k+1)},\param^*-\param^{(k+1)}\rangle \\
    &\qquad\qquad + \langle v^{(k+1)},\param^*-\param^{(k)}\rangle + \tau^{(k)} \langle g^{(k)},\param^*-\param^{(k)}\rangle \\
    &= \func_\delta(\param^{(k)})-\func_\delta(\param^{(k+1)}) - \langle v^{(k+1)},\param^{(k)}-\param^{(k+1)}\rangle + \tau^{(k)}\langle g^{(k)},\param^*-\param^{(k)}\rangle \\
    &= D_{\func_\delta}^{v^{(k+1)}}(\param^{(k)},\param^{(k+1)}) + \tau^{(k)}\langle g^{(k)},\param^*-\param^{(k)}\rangle.
\end{align*}
Taking expectations and \revision{using that $g^{(k)}$ and $\param^* - \param^{(k)}$ are stochastically independent} to replace $g^{(k)}$ with $\nabla\mathcal{L}(\param^{(k)})$ inside the expectation concludes the proof.
Note that this argument does not apply to the first equality since $\param^{(k+1)}$ is \emph{not stochastically independent} of $\param^{(k)}$.
\end{proof}

Now we prove \cref{thm:cvgc_norm} by showing that the Bregman distance to the minimizer of the loss and that the \revision{iterates converge in norm to the minimizer}.

\begin{proof}[Proof of \cref{thm:cvgc_norm}]
Using \labelcref{ineq:gradLip},
\cref{ass:mu-convex}, and Young's inequality we obtain for any $c>0$
\begin{align*}
    &\phantom{=}\langle \nabla\empLoss(\param^{(k)}), \param^* - \param^{(k+1)}\rangle \\
    &=\langle \nabla\empLoss(\param^{(k+1)}), \param^* - \param^{(k+1)}\rangle + \langle \nabla\empLoss(\param^{(k)})-\nabla\empLoss(\param^{(k+1)}), \param^* - \param^{(k+1)}\rangle \\
    &\leq \empLoss(\param^*) - \empLoss(\param^{(k+1)}) - \frac{\mu}{2}\norm{\param^*-\param^{(k+1)}}^2 + \frac{L^2}{2c}\norm{\param^{(k)}-\param^{(k+1)}}^2 + \frac{c}{2}\norm{\param^*-\param^{(k+1)}}^2.
\end{align*}
Using that $\empLoss(\param^*)\leq\empLoss(\param^{(k+1)})$ we obtain
\begin{align*}
    \langle \nabla\empLoss(\param^{(k)}), \param^* - \param^{(k+1)}\rangle \leq \frac{L^2}{2c}\norm{\param^{(k)}-\param^{(k+1)}}^2 + \frac{c-\mu}{2}\norm{\param^*-\param^{(k+1)}}^2.
\end{align*}
Plugging this into the expression \labelcref{eq:breg_diff_1} for $d_{k+1}-d_k$ yields
\begin{align*}
d_{k+1}-d_k &=
-\E\left[D_{\func_\delta}^{v^{(k)}}(\param^{(k+1)},\param^{(k)})\right] + \tau^{(k)}
\Exp{\langle \nabla\empLoss(\param^{(k)}), \param^* - \param^{(k+1)}\rangle}\\
&\qquad\qquad + \tau^{(k)}\Exp{\langle g^{(k)}-\nabla\empLoss(\param^{(k)}), \param^* - \param^{(k+1)}\rangle} \\
&\leq -\E\left[D_{\func_\delta}^{v^{(k)}}(\param^{(k+1)},\param^{(k)})\right] + \tau^{(k)}\frac{L^2}{2c}\E\left[\norm{\param^{(k)}-\param^{(k+1)}}^2\right] \\
&\qquad\qquad + \frac{c-\mu}{2}\tau^{(k)}\E\left[\norm{\param^*-\param^{(k+1)}}^2\right] \\
&\qquad\qquad +\tau^{(k)}\Exp{\langle g^{(k)}-\nabla\empLoss(\param^{(k)}), \param^* - \param^{(k+1)}\rangle}.
\end{align*}
Now we utilize that $\param^*-\param^{(k)}$ and $g^{(k)}-\nabla\empLoss(\param^{(k)})$ are stochastically independent and that the latter has zero expectation to infer
\begin{align*}
    &\phantom{=}\Exp{\langle g^{(k)}-\nabla\empLoss(\param^{(k)}), \param^* - \param^{(k+1)}\rangle}  \\
    &= \Exp{\langle g^{(k)}-\nabla\empLoss(\param^{(k)}),  \param^{(k)} - \param^{(k+1)}\rangle} \\
    &\qquad\qquad + \Exp{\langle g^{(k)}-\nabla\empLoss(\param^{(k)}),  \param^* - \param^{(k)}\rangle} \\
    &= \Exp{\langle g^{(k)}-\nabla\empLoss(\param^{(k)}),  \param^{(k)} - \param^{(k+1)}\rangle}.
\end{align*}
Applying Young's inequality and using \cref{ass:variance} yields
\begin{align*}
  &\phantom{\leq} \tau^{(k)}  \Exp{\langle g^{(k)}-\nabla\empLoss(\param^{(k)}), \param^{(k)} - \param^{(k+1)}\rangle} \\
  &\leq  \frac{\delta(\tau^{(k)})^2}2 \sigma + \frac{\sigma}{2\delta}\Exp{\norm{\param^{(k)} - \param^{(k+1)}}^2}.
\end{align*}
Plugging this into the expression for $d_{k+1}-d_k$ and reordering we infer that for \revision{$\tau^{(k)}\leq\frac{\mu}{2\delta L^2}$} and $c=\mu/2$ it holds
\begin{align*}
&\phantom{\leq} d_{k+1} - d_k + \frac{\mu}{4}\tau^{(k)}\E\left[\norm{\param^*-\param^{(k+1)}}^2\right] \\
&\leq -\E\left[\revision{D_{\func_\delta}^{v^{(k)}}}(\param^{(k+1)},\param^{(k)})\right]  
+\frac{\delta(\tau^{(k)})^2}2 \sigma + \frac{\sigma}{2\delta}\Exp{\norm{\param^{(k)} - \param^{(k+1)}}^2},
\end{align*}
\revision{which yields \labelcref{ineq:decay_bregman_distance}}.
Summing this inequality, using \cref{cor:square_sum}\revision{---which is possible since $\tau^{(k)}\leq\frac{\mu}{2\delta L^2}\leq \frac{1}{2\delta L}\leq \frac{2}{\delta L}$---}, and that the step sizes are square-summable then shows that
\begin{align*}
    \sum_{k=0}^\infty\tau^{(k)}\E\left[\norm{\param^*-\param^{(k+1)}}^2\right] < \infty.
\end{align*}
Since $\sum_{k=0}^\infty \tau^{(k)}=\infty$ this means that 
\begin{align*}
    \min_{k\in\{1,\dots,K\}}\E\left[\norm{\param^*-\param^{(k)}}^2\right]\to 0,\quad K\to\infty.
\end{align*}
Hence, for an appropriate subsequence $\param^{(k_j)}$ it holds
\begin{align*}
    \lim_{j\to\infty}\E\left[\norm{\param^*-\param^{(k_j)}}^2\right]=0.
\end{align*}
\end{proof}

\begin{proof}[Proof of \cref{cor:convergence_finite_d}]
Since $\func$ is equal to the $\ell_1$-norm, $\func$ admits the triangle inequality $\func(\param^*)-\func(\param^{(k)})\leq\func(\param^*-\param^{(k)})$.
Furthermore, since $d:=\dim\Param$ is finite, it admits the norm inequality $\func(\param)\leq \sqrt{d}\norm{\param}$ and has bounded subgradients $\norm{\sg}=\norm{\sign(\param)}\leq d$ for all $\param\in\Param$.

Hence, using \cref{lem:sum_breg_dist} we can estimate
\begin{align*}
    d_k &= \Exp{D_{\func_\delta}^{v^{(k)}}(\param^*,\param^{(k)})} = \Exp{D_\func^{\sg^{(k)}}(\param^*,\param^{(k)})} + \frac{1}{2\delta}\Exp{\norm{\param^*-\param^{(k)}}^2}\\
    &=\Exp{\func(\param^*)-\func(\param^{(k)})-\langle\sg^{(k)},\param^*-\param^{(k)}\rangle}+ \frac{1}{2\delta}\Exp{\norm{\param^*-\param^{(k)}}^2}\\
    &\leq (\sqrt{d}+d)\Exp{\norm{\param^*-\param^{(k)}}}+ \frac{1}{2\delta}\Exp{\norm{\param^*-\param^{(k)}}^2}.
\end{align*}
Hence, since $\Exp{\norm{\param^*-\param^{(k_j)}}^2}$ converges to zero according to \cref{thm:cvgc_norm}, the same is true for the sequence $d_{k_j}$.
Furthermore, we have proved that $d_{k+1}-d_k\leq c_k$, where $c_k$ is a non-negative and summable sequence.
This also implies $d_m \leq d_k + \sum_{j=k}^\infty c_k$ for every $m>k$.

Since $c_k$ is summable and $d_{k_j}$ converges to zero there exists $k\in\N$ and $l\in\N$ such that
\begin{align*}
    \sum_{j=k}^\infty c_j < \frac{\varepsilon}{2},\quad
    d_{k_l} < \frac{\varepsilon}{2},\quad
    k_l > k.
\end{align*}
Hence, we obtain for any $m>k_l$
\begin{align*}
    d_m \leq d_{k_l} + \sum_{j=k_l}^\infty c_k \leq
    d_{k_l} + \sum_{j=k}^\infty c_k < \varepsilon.
\end{align*}
Since $\varepsilon>0$ was arbitrary, this implies that $d_m\to 0$ as $m\to\infty$.
\end{proof}

Finally we prove \cref{thm:cvgc_breg_dist} based on \cref{ass:breg-convex}, which asserts convergence in the Bregman distance.

\revision{
\begin{proof}[Proof of \cref{thm:cvgc_breg_dist}]
The proof goes along the lines of \citet{turinici2021convergence}, however \cref{ass:variance} is weaker than the assumption posed there.
\\
\textbf{Item 1:} 
Since proximal operators are 1-Lipschitz it holds 
$$\norm{\param^{(k+1)}-\param^{(k)}}=\norm{\prox{\delta\func}(\delta v^{(k+1)}) - \prox{\delta\func}(\delta v^{(k)})}\leq\delta\norm{v^{(k+1)}-v^{(k)}}.$$ 
Using the Cauchy-Schwarz inequality and this estimate yields
\begin{align*}
    \frac{1}{\tau^{(k)}}\E\left[D_{\func_\delta}^{v^{(k+1)}}(\param^{(k)},\param^{(k+1)})\right]
    &\leq \frac{1}{\tau^{(k)}}\Exp{D_{\func_\delta}^\mathrm{sym}(\param^{(k)},\param^{(k+1)})} \\
    &=\frac{1}{\tau^{(k)}}\Exp{\langle v^{(k+1)}-v^{(k)}, \param^{(k+1)}-\param^{(k)}\rangle}\\
    &= -\Exp{\langle g^{(k)}, \param^{(k+1)}-\param^{(k)}\rangle} \\
    &\leq \sqrt{\Exp{\norm{g^{(k)}}^2}}\sqrt{\Exp{\norm{\param^{(k+1)}-\param^{(k)}}^2}} \\
    &\leq \delta \sqrt{\Exp{\norm{g^{(k)}}^2}} \sqrt{\Exp{\norm{v^{(k+1)}-v^{(k)}}^2}}\\
    &= \delta \tau^{(k)}\Exp{\norm{g^{(k)}}^2}.
\end{align*}
It can be easily seen that \cref{ass:variance} is equivalent to
\begin{align*}
    \Exp{\norm{g^{(k)}}^2} 
    \leq \sigma^2 + \Exp{\norm{\nabla\empLoss(\param^{(k)})}^2},
\end{align*}
which together with the Lipschitz continuity of $\nabla\empLoss$ from \cref{ass:loss} implies
\begin{align*}
    \Exp{D_{\func_\delta}^{v^{(k+1)}}(\param^{(k)},\param^{(k+1)})}
    &\leq \delta(\tau^{(k)})^2\left(\sigma^2 + \Exp{\norm{\nabla\empLoss(\param^{(k)})}^2}\right) \\
    &\leq \delta(\tau^{(k)})^2\left(\sigma^2 + L^2 \Exp{\norm{\param^*-\param^{(k)}}^2}\right) \\
    &\leq \delta(\tau^{(k)})^2\left(\sigma^2 + 2\delta L^2 d_k\right).
\end{align*}
We plug this estimate into \labelcref{eq:breg_diff_2} and utilize \cref{ass:breg-convex} to obtain
\begin{align*}
    d_{k+1} - d_k 
    &\leq \delta(\tau^{(k)})^2\left(\sigma^2 + 2\delta L^2 d_k\right) + \tau^{(k)}\E\left[\mathcal{L}(\param^*)-\mathcal{L}(\param^{(k)}) - \nu D_{\func_\delta}^{v^{(k)}}(\param^*,\param^{(k)})\right] \\
    &\leq \delta(\tau^{(k)})^2\left(\sigma^2 + 2\delta L^2 d_k\right) - \tau^{(k)}\nu\E\left[D_{\func_\delta}^{v^{(k)}}(\param^*,\param^{(k)})\right]\\
    &= \delta(\tau^{(k)})^2\left(\sigma^2 + 2\delta L^2 d_k\right) - \tau^{(k)}\nu d_k,
\end{align*}
which is equivalent to \eqref{eq:estimate_expects}.
\\
\textbf{Item 2:}
For $\tau^{(k)}$ sufficiently small \eqref{eq:estimate_expects} implies
\begin{align*}
    d_{k+1}\leq \left(1-\tau^{(k)}\tilde\nu\right)d_k + \delta(\tau^{(k)})^2\sigma^2
\end{align*}
for a constant $\tilde\nu\in(0,\nu)$.
For any $C\in\R$ we can reformulate this to
\begin{align}\label{eq:reform}
    d_{k+1} - C 
    \leq 
    (1-\tau^{(k)}\tilde\nu)(d_k - C) + \tau^{(k)} \left(\tau^{(k)}\delta \sigma^2 - \tilde\nu C\right).
\end{align}
If $\tau^{k}=\tau$ for all $k\in\N$ is constant and we choose $C=\tau\tfrac{\delta \sigma^2}{\tilde\nu}$, we obtain
\begin{align*}
    \left(d_{k+1} - \tau\tfrac{\delta \sigma^2}{\tilde\nu}\right)_+ \leq (1-\tau\tilde\nu)\left(d_k - \tau\tfrac{\delta \sigma^2}{\tilde\nu}\right)_+,
\end{align*}
where we passed to the positive part $x_+:=\max(x,0)$.
Iterating this inequality yields
\begin{align*}
    \left(d_{k+n} - \tau\tfrac{\delta \sigma^2}{\tilde\nu}\right)_+ \leq (1-\tau\tilde\nu)^n\left(d_k - \tau\tfrac{\delta \sigma^2}{\tilde\nu}\right)_+
\end{align*}
and hence for $\tau<\tfrac1{\tilde\nu}$ we get
\begin{align*}
    \limsup_{k\to\infty} d_k \leq \tau\tfrac{\delta \sigma^2}{\tilde\nu}.
\end{align*}
Demanding $\tau<\tfrac{\eps\tilde\nu}{\delta\sigma^2} \wedge \tfrac{1}{\tilde\nu}$ we finally obtain
\begin{align*}
    \limsup_{k\to\infty} d_k \leq \eps.
\end{align*}
\\
\textbf{Item 3:}
We use the reformulation \labelcref{eq:reform} with $C=\eps>0$.
Since $\tau^{(k)}$ converges to zero the second term is non-positive for $k\in\N$ sufficiently large and we obtain
\begin{align*}
    (d_{k+1}-\eps)_+ \leq (1-\tau^{(k)}\tilde\nu)(d_k-\eps)_+. 
\end{align*}
Iterating this inequality yields
\begin{align*}
    (d_{k+n}-\eps)_+ \leq \prod_{l=k}^{k+n-1}(1-\tau^{(l)}\tilde\nu)(d_k-\eps)_+.
\end{align*}
If $k\in\N$ is sufficiently large, then $\tau^{(l)}\tilde\nu < 1$ for all $l\geq k$ and the product satisfies
\begin{align*}
    \prod_{l=k}^{k+n-1}(1-\tau^{(l)}\tilde\nu) = \exp\left(\sum_{l=k}^{k+n-1}\log(1-\tau^{(l)}\tilde\nu)\right)\leq 
    \exp\left(-\sum_{l=k}^{k+n-1}\tau^{(l)}\tilde\nu\right)\to 0,\quad n\to\infty,
\end{align*}
where we used $\log(1-x)\leq-x$ for all $x<1$ and $\sum_k \tau^{(k)}=\infty$.
Since $\eps$ was arbitrary we obtain the assertion.
\end{proof}
}

\end{document}